%% file: main.tex
\documentclass{article}

\usepackage{microtype}
\usepackage{graphicx}
\usepackage{subfigure}
\usepackage{booktabs} 
\usepackage{wrapfig}
\usepackage{multirow}
\usepackage{amsthm}
\usepackage{amsmath,amsfonts,bm}
\usepackage[ruled,vlined]{algorithm2e}
\newtheorem{theorem}{Theorem}
\newcommand{\method}{GSMRL\xspace}
\newcommand{\KL}{D_{\mathrm{KL}}}

\usepackage{hyperref}

\usepackage[accepted]{icml2021}

\icmltitlerunning{Active Feature Acquisition with Generative Surrogate Models}

\begin{document}

\twocolumn[
\icmltitle{Active Feature Acquisition with Generative Surrogate Models}



\icmlsetsymbol{equal}{*}

\begin{icmlauthorlist}
\icmlauthor{Yang Li}{unc}
\icmlauthor{Junier B. Oliva}{unc}
\end{icmlauthorlist}

\icmlaffiliation{unc}{Department of Computer Science, University of North Carolina at Chapel Hill, NC, USA}

\icmlcorrespondingauthor{Yang Li}{yangli95@cs.unc.edu}

\icmlkeywords{Machine Learning, ICML}

\vskip 0.3in
]



\printAffiliationsAndNotice{}  

\begin{abstract}
Many real-world situations allow for the acquisition of additional relevant information when making an assessment with limited or uncertain data. However, traditional ML approaches either require all features to be acquired beforehand or regard part of them as missing data that cannot be acquired. In this work, we consider models that perform active feature acquisition (AFA) and query the environment for unobserved features to improve the prediction assessments at evaluation time. 
Our work reformulates the Markov decision process (MDP) that underlies the AFA problem as a generative modeling task and optimizes a policy via a novel model-based approach. We propose learning a generative surrogate model (GSM) that captures the dependencies among input features to assess potential information gain from acquisitions. The GSM is leveraged to provide intermediate rewards and auxiliary information to aid the agent navigate a complicated high-dimensional action space and sparse rewards.
Furthermore, we extend AFA in a task we coin \emph{active instance recognition} (AIR) for the unsupervised case where the target variables are the unobserved features themselves and the goal is to collect information for a particular instance in a cost-efficient way. Empirical results demonstrate that our approach achieves considerably better performance than previous state of the art methods on both supervised and unsupervised tasks.
\end{abstract}

\section{Introduction}
A typical machine learning paradigm for discriminative tasks is to learn the distribution of an output, $y$ given a complete set of features, $x \in \mathbb{R}^d$: $p(y \mid x)$. Although this paradigm is successful in a multitude of domains, it is incongruous with the expectations of many real-world intelligent systems in two key ways: first, it assumes that a complete set of features has been observed; second, as a consequence, it also assumes that no additional information (features) of an instance may be obtained at evaluation time. These assumptions often do not hold; human agents routinely reason over instances with incomplete data and decide when and what additional information to obtain. For example, consider a doctor diagnosing a patient. The doctor usually has not observed all possible measurements (such as blood samples, x-rays, etc.) for the patient. He/she is not forced to make a diagnosis based on the observed measurements; instead, he/she may dynamically decide to take more measurements to help determine the diagnosis. Of course, the next measurement to make (feature to observe), if any, will depend on the values of the already observed features; thus, the doctor may determine a different set of features to observe from patient to patient (instance to instance) depending on the values of the features that were observed. Hence, each patient will not have the same subset of features selected (as would be the case with typical feature selection). Furthermore, acquiring features typically involves some cost (in time, money and risk), and intelligent systems are expected to automatically balance the cost and improvement on performance.
In order to more closely match the needs of many real-world applications, we propose an active feature acquisition (AFA) model that not only makes predictions with incomplete/missing features, but also determines the next feature that would be the most valuable to obtain for a particular instance.

\begin{figure*}
    \centering
    \begin{minipage}{0.44\linewidth}
    \centering
    \includegraphics[width=0.98\textwidth]{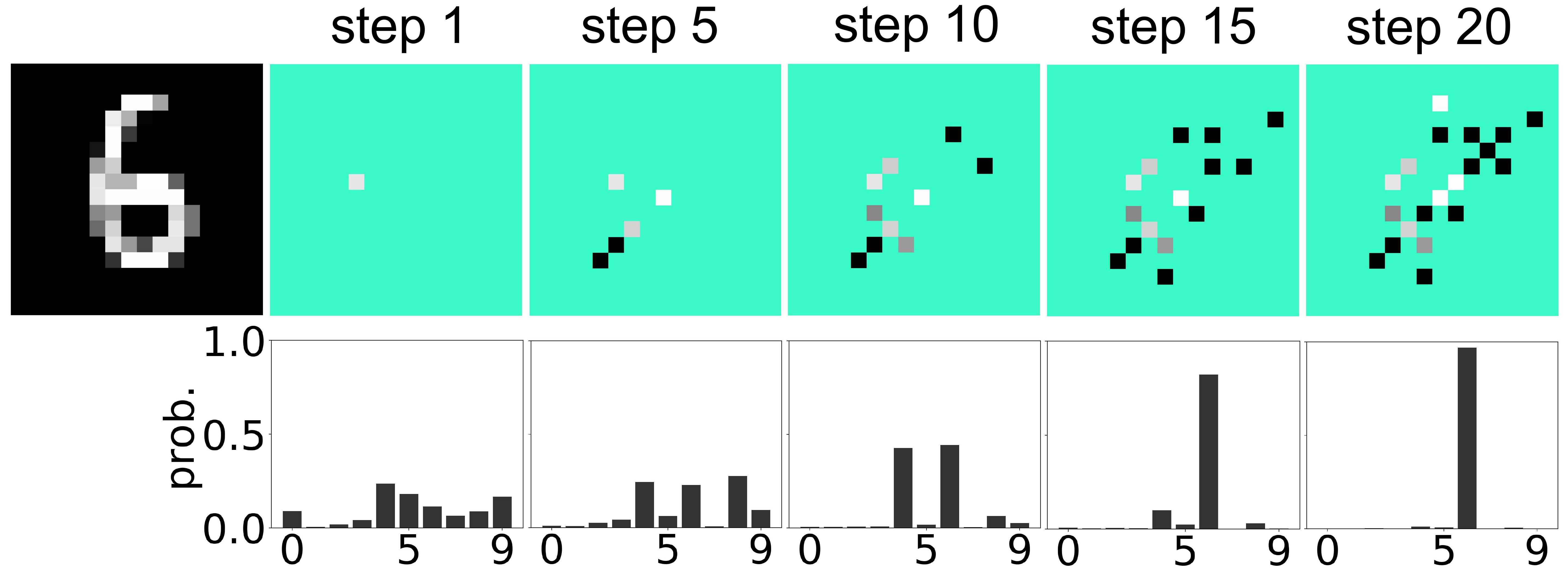}
    \vspace{-8pt}
    \caption{Active feature acquisition on MNIST. Example of our acquisition process (top) and the prediction probabilities (bottom). The green masks indicate the unobserved features.}
    \label{fig:mnist_cls_acq}
    \end{minipage}
    \hspace{35pt}
    \begin{minipage}{0.44\linewidth}
    \centering
    \includegraphics[width=0.98\textwidth]{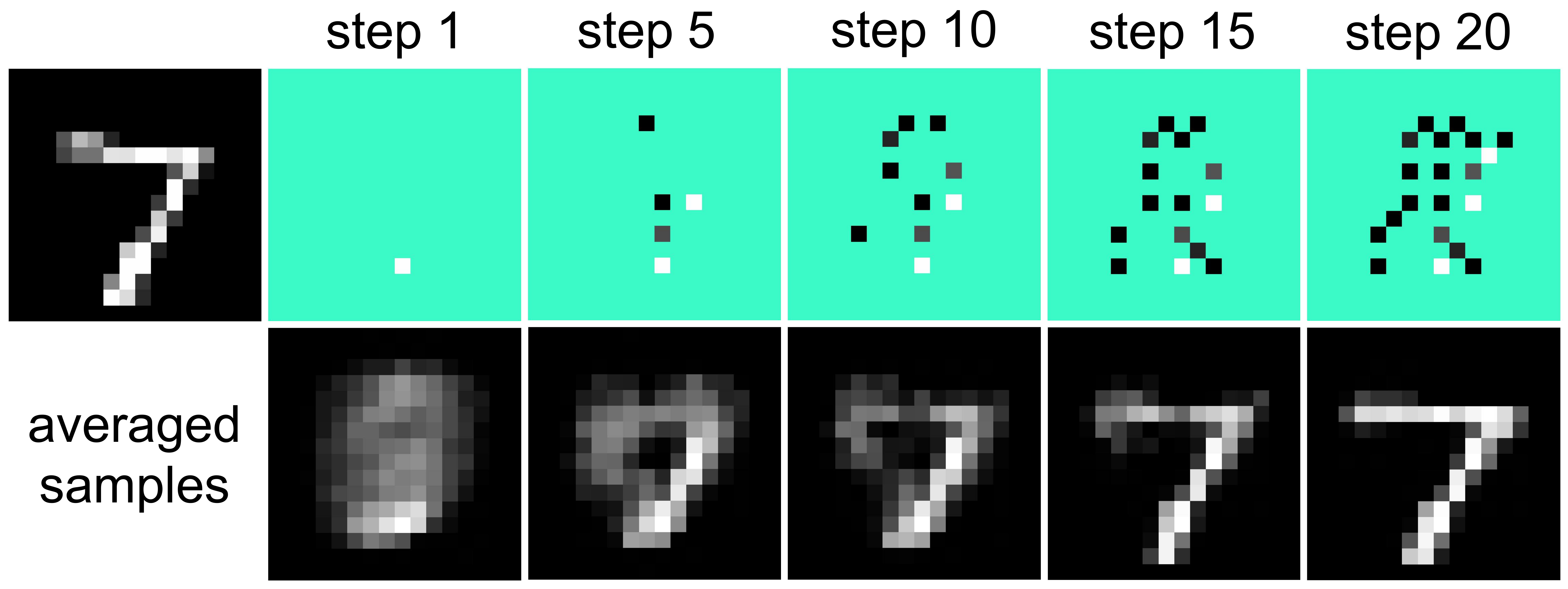}
    \vspace{-8pt}
    \caption{Active instance recognition on MNIST. Example of our acquisition process (top) and averaged inpaintings (bottom). The green masks indicate the unobserved features.}
    \label{fig:mnist_air_acq}
    \end{minipage}
    \vspace{-8pt}
\end{figure*}

As noted in \cite{shim2018joint}, the active feature acquisition problem may be formulated as a Markov decision process (MDP), where the state is the set of currently observed features and the action is the next feature to acquire. A special action indicates whether to stop the acquisition process and make a final prediction. After acquiring its value and paying the acquisition cost, the newly acquired feature is added to the observed subset and the agent proceeds to the next acquisition step. Once the agent decides to terminate the acquisition, it makes a final prediction based on the features acquired thus far. For example, in an image classification task (Fig.~\ref{fig:mnist_cls_acq}), the agent would dynamically acquire pixels until it is certain of the image class. The goal of the agent is to maximize the prediction performance while minimizing the acquisition cost.

The key insight of this work is that the dynamics model for the AFA MDP is based on the conditionals of the features: $p(x_j \mid x_o)$, where $x_j$ is an unobserved feature selected for acquisition and $x_o$ are the previously acquired features. Thus, we develop a model-based approach through generative modeling of \emph{all} conditional dependencies. Equipped with the surrogate model, our method, \emph{Generative Surrogate Models for RL} (\method), essentially combines model-free and model-based RL into a holistic framework.

\method rectifies several short-comings of a model-free scheme such as JAFA \cite{shim2018joint}. In the aforementioned MDP, the agent pays the acquisition cost at each acquisition step but only receives a reward about the prediction after completing the acquisition process. To reduce the sparsity of the rewards and simplify the credit assignment problem for potentially long episodes \citep{minsky1961steps,sutton1988learning}, we leverage a surrogate model to provide intermediate rewards
by assessing the information gain of a newly acquired feature, which quantifies how much our confidence about the prediction improves by acquiring this feature. 
In addition to sparse rewards, an AFA agent must also navigate a complicated high-dimensional action space \cite{dulac2015deep}, and must manage multiple roles as it has to: implicitly model dependencies, perform a cost/benefit analysis, and act as a classifier. To lessen the burden, 
we also propose using the surrogate model to provide side information that assists the agent. 
The side information shall explicitly inform the agent of: 1) uncertainty and imputations for unobserved features; 2) an estimate of the expected information gain of future acquisitions; 3) uncertainty of the target output.
This allows the agent to easily assess its current uncertainty and helps the agent `look ahead' to expected outcomes from future acquisitions.

In this work, we also propose the first (to the best of our knowledge) unsupervised AFA task, which we coin \emph{active instance recognition} (AIR).
Here we consider the case where there is not a single target variable, but instead the target of interest may be the remaining unobserved features themselves.
That is, rather than reducing the uncertainty with respect to some desired output response (that cannot be directly queried and must be predicted), the task is to query as few features as possible that allows the agent to correctly uncover the remaining unobserved features. 
For example, in image data AIR, an agent queries new pixels until it can reliably uncover the remaining pixels (see Fig.~\ref{fig:mnist_air_acq}).
AIR is especially relevant in surveying tasks, which are broadly applicable across various domains and applications.
Most surveys aim to discover a broad set of underlying characteristics of instances (e.g., citizens in a census) using a limited number of queries (questions in the census form), which is at the core of AIR.
Policies for AIR would build a personalized subset of survey questions (for individual instances) that quickly uncovered the likely answers to all remaining questions.

Our contributions are as follows: 1) We reformulate the AFA problem as a generative modeling task and build surrogate models that capture the state transitions with arbitrary conditional distributions. 2) We develop methodology to leverage the surrogate model to provide intermediate rewards as training signals and to provide auxiliary information that assists the agent. Our framework represents a novel combination of model-free and model-based RL. 3) We propose the first unsupervised active feature acquisition task where the target variables are the unobserved features themselves. 4) We achieve state-of-the-art performance on both supervised and unsupervised tasks in the largest scale AFA study to date. 5) We open-source a standardized environment inheriting the OpenAI gym interfaces \citep{1606.01540} to assist future research on active feature acquisition. Code will be released upon publication.

\section{Methods}
In this section, we first describe our \method~framework for both active feature acquisition (AFA) and active instance recognition (AIR) problems. We then develop our RL algorithm and the corresponding surrogate models for different settings. We also introduce a special application that acquires features for time series data.

\subsection{AFA and AIR with \method}
Consider a discriminative task with features $x \in \mathbb{R}^d$ and target $y$. Instead of predicting the target by first collecting all the features, we perform a sequential feature acquisition process in which we start from an empty set of features and actively acquire more features. There is typically a cost associated with features and the goal is to maximize the task performance while minimizing the acquisition cost, i.e.,
\begin{equation}\label{eq:goal}
    \text{minimize}~ \mathcal{L}(\hat{y}(x_o), y) + \alpha \mathcal{C}(o),
\end{equation}
where $\mathcal{L}(\hat{y}(x_o), y)$ represents the loss function between the prediction $\hat{y}(x_o)$ and the target $y$. Note that the prediction is made with the acquired feature subset $x_o, o \subseteq \{1,\ldots,d\}$. Therefore the agent should be able to predict with arbitrary subsets. $\mathcal{C}(o)$ represents the cost of the acquired features $o$. The hyperparameter $\alpha$ controls the trade-off between prediction loss and acquisition cost. For unsupervised tasks, the target variable $y$ equals to $x$; that is, we acquire features actively to represent the instance with a selected subset. 

In order to solve the optimization problem in \eqref{eq:goal}, we formulate it as a Markov decision process as in \citep{zubek2004pruning,shim2018joint}:
\begin{equation}
\begin{aligned}
    &s = [o,x_o], \quad a \in u \cup \phi, \\ 
    &r(s,a) = -\mathcal{L}(\hat{y}, y) \mathbb{I}(a=\phi) - \alpha \mathcal{C}(a) \mathbb{I}(a \neq \phi).
\end{aligned}
\end{equation}
The state $s$ is the current acquired feature subset $o \subseteq \{1,\ldots,d\}$ and their values $x_o$. The action space contains the remaining candidate features $u=\{1,\ldots,d\} \setminus o$ and a special action $\phi$ that indicates the termination of the acquisition process. To optimize the MDP, a reinforcement learning agent acts based on the observed state and receives rewards from the environment. When the agent acquires a new feature $i$, the current state transits to a new state following
$o \xrightarrow{i} o \cup i, x_o \xrightarrow{i} x_o \cup x_i$,
and the reward is the negative acquisition cost of this feature. $x_i$ is obtained from the environment (i.e. we observe the true $i^\mathrm{th}$ feature value for the instance). 

\textbf{Feature Dependencies as Dynamics Model} A surprisingly unexplored property of the AFA MDP, and the driving observation to our work, is that the dynamics of the problem are dictated by \emph{conditional dependencies among the data's features}. That is, the state transitions are based on the conditionals: $p(x_j \mid x_o)$, where $x_j$ is an unobserved feature. 
Therefore we frame our approach according to the estimation of conditionals among features with generative models.
We build a surrogate model to learn the distribution $p(y, x_j \mid x_o)$, where $x_j$ and $x_o$ contain arbitrary features from $x$. We find that the most efficacious use of our generative surrogate model (see section \ref{sec:ablations}) is a hybrid model-based approach that utilizes intermediate rewards and side information stemming from dependencies.

\textbf{Intermediate Rewards} When the agent terminates the acquisition and makes a prediction, the reward equals to the negative prediction loss using current acquired features. Since the prediction is made at the end of acquisitions, the reward of the prediction is received only when the agent decides to terminate the acquisition process. This is a typical temporal credit assignment problem, which may affect the learning of the agent \citep{minsky1961steps,sutton1988learning}.
Given the surrogate model, we propose to remedy the credit assignment problem by providing intermediate rewards for each acquisition. Inspired by the information gain, the surrogate model assesses the intermediate reward for a newly acquired feature $i$ with
\begin{equation}\label{eq:info_gain}
    r_m(s,i) = H(y \mid x_o) - \gamma H(y \mid x_o, x_i),
\end{equation}
where $\gamma$ is a discount factor for the MDP. In appendix \ref{sec:invariance}, we show that our intermediate rewards will not change the optimal policy.

\textbf{Side Information} In addition to intermediate rewards, we propose using the surrogate model to also provide side information to assist the agent, which includes the current prediction and output likelihood, the possible values and corresponding uncertainties of the unobserved features, and the estimated utilities of the candidate acquisitions. The current prediction $\hat{y}$ and likelihood $p(y \mid x_o)$ inform the agent about its confidence, which can help the agent determine whether to stop the acquisition. The imputed values and uncertainties of the unobserved features give the agent the ability to look ahead into and future and guide its exploration. For example, if the surrogate model is very confident about the value of a currently unobserved feature, then acquiring it would be redundant. The utility of a feature $i$ is estimated by its \emph{expected} information gain to the target variable:
\begin{equation}\label{eq:utility}
\begin{aligned}
    \mathcal{U}_i &= H(y \mid x_o) - \mathbb{E}_{p(x_i \mid x_o)}H(y \mid x_i, x_o)\\
        &= H(x_i \mid x_o) - \mathbb{E}_{p(y \mid x_o)}H(x_i \mid y, x_o),
\end{aligned}
\end{equation}
where the surrogate model is used to estimate the expected entropies. The utility essentially quantifies the conditional mutual information $I(x_i; y \mid x_o)$ between each candidate feature and the target variable. A greedy policy can be easily built based on the utilities where the next feature to acquire is the one with maximum utility \citep{ma2018eddi,gong2019icebreaker}. Here, our agent takes the utilities as side information to help balance exploration and exploitation, and eventually learns a non-greedy policy.

\begin{algorithm}[tb]
\caption{Active Feature Acquisition with \method}
\label{alg:RL}
\begin{algorithmic}
\INPUT{pretrained surrogate model $M$; agent $\textit{agent}$; prediction model $f_\theta(\cdot)$; test dataset $D$ to be acquired}
\STATE{1. instantiate an environment: $\textit{env}=\text{Env}(D)$}
\STATE{2. $x_o$, o = \textit{env}.reset()}
\STATE{3. done = False; reward = 0}
\WHILE{not done}{
  \STATE{aux = $M$.query($x_o$, o)}
  \STATE{action = \textit{agent}.act($x_o$, o, aux)}
  \STATE{$r_m =$ $M$.reward($x_o$, $o$, action)}
  \STATE{$x_o$, o, done, r = \textit{env}.step(action)}
  \STATE{reward = reward + r + $r_m$}}
\ENDWHILE
\STATE{prediction = $\textit{agent}$.predict($x_o$, o, aux)}
\end{algorithmic}
\end{algorithm}

\textbf{Prediction Model} When the agent deems that acquisition is complete, it makes a final prediction based on the acquired features thus far. The final prediction may be made using the surrogate model, i.e., $p(y \mid x_o)$, but it might be beneficial to train predictions specifically based on the agent's own distribution of acquired features $o$, since the surrogate model is agnostic to the feature acquisition policy of the agent. Therefore, we build a prediction model $f_\theta(\cdot)$ that takes both the current state $x_o$ and the side information as inputs (i.e.~the same inputs as the policy). The prediction model can be trained simultaneously with the policy as an auxiliary task. Weight sharing between the policy and prediction function facilitates the learning of better representations. 
\begin{wrapfigure}{r}{0.48\linewidth}
    \centering
    \vspace{-5pt}
    \includegraphics[width=\linewidth]{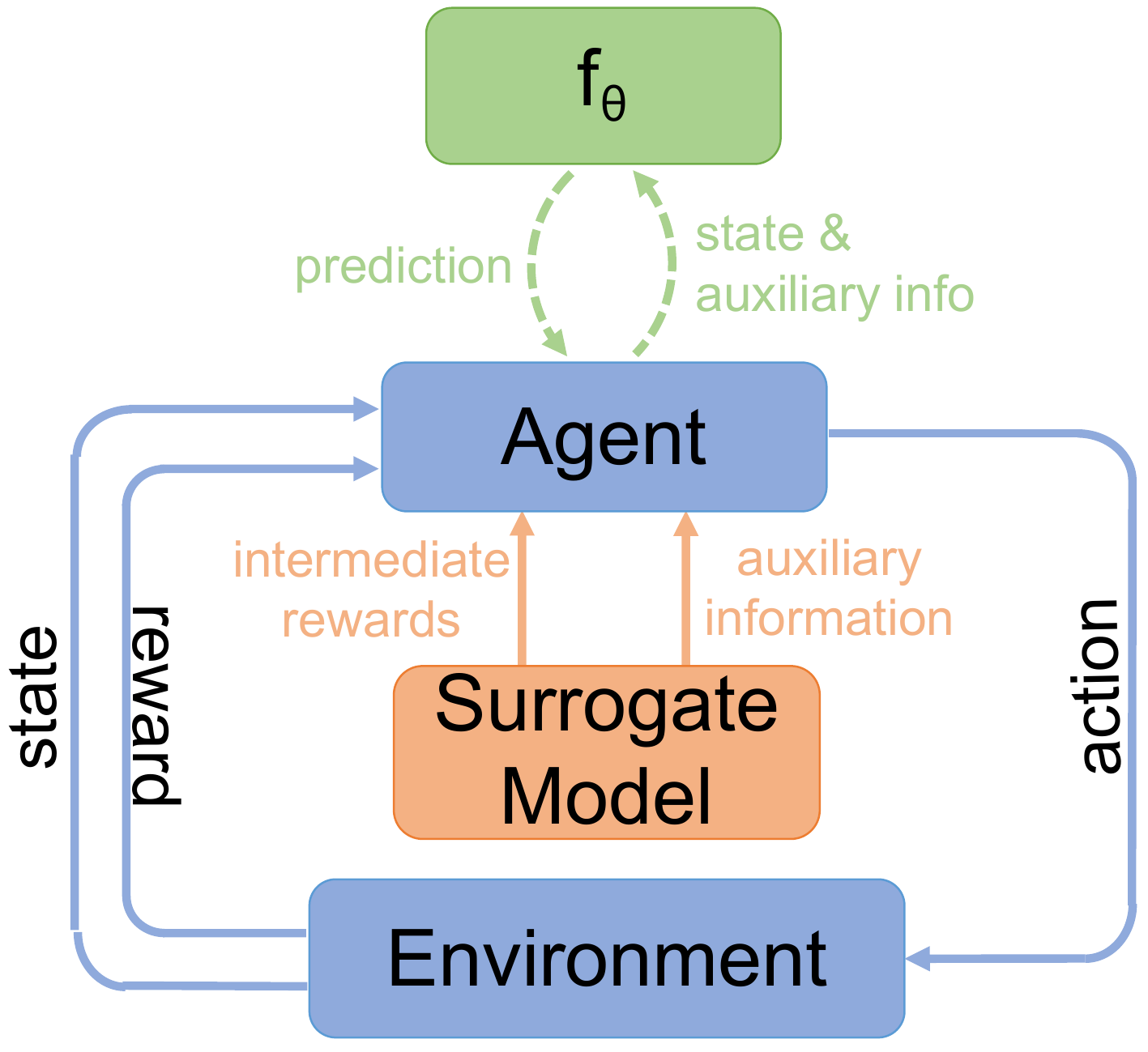}
    \vspace{-15pt}
    \caption{Illustration of our \method~framework with a prediction model $f_\theta$.}
    \label{fig:model}
    \vspace{-8pt}
\end{wrapfigure}
Given the two predictions from the surrogate model and the prediction model respectively, the final reward $-\mathcal{L}(\hat{y}, y)$ during training is the maximum one using either predictions. During test time, we choose one prediction based on validation performance. An illustration of our framework is presented in Fig.~\ref{fig:model}. Please refer to Algorithm~\ref{alg:RL} for pseudo-code of the acquisition process with our \method~framework. Please also see Algorithm~\ref{alg:RL_supp} in the appendix for a detailed version. We will expound on the surrogate models for different settings below.

\subsubsection{Surrogate Model for AFA}
As we mentioned above, the surrogate model learns the conditional distributions $p(y, x_j \mid x_o)$. Note that $x_o$ is an arbitrary subset of the features and $x_j$ is an arbitrary unobserved feature since the surrogate model must be able to assist arbitrary policies, and acquired features will vary from instance to instance. Thus, there are 
an exponential number of different conditionals that the surrogate model must estimate for a $d$-dimensional feature space. Therefore, learning a separate model for each different conditional is intractable. Fortunately, \citet{ivanov2018variational} and \citet{li2019flow} have proposed models to learn arbitrary conditional distributions $p(x_u \mid x_o)$. They regard different conditionals as different tasks and train VAE and normalizing flow based generative models, respectively, in a multi-task fashion to capture the arbitrary conditionals with a unified model. 
In this work, we leverage arbitrary conditionals and extend them to model the target variable $y$ as well.
For continuous target variables, we concatenate them with the features, thus $p(y, x_j \mid x_o)$ can be directly modeled. For discrete target variables, where we have a mix of continuous features and discrete labels, we use Bayes' rule:
\begin{equation}\label{eq:discrete}
    p(y,x_j \mid x_o) = \frac{p(x_j \mid y, x_o)p(x_o \mid y)P(y)}{\sum_{y'}p(x_o \mid y')P(y')}.
\end{equation}
We employ a variant arbitrary conditioning model that conditions on the target $y$ to obtain the arbitrary conditional likelihoods $p(x_j \mid y, x_o)$ and $p(x_o \mid y)$ in \eqref{eq:discrete}.

Given a trained surrogate model, the prediction $p(y \mid x_o)$, the information gain in \eqref{eq:info_gain}, and the utilities in \eqref{eq:utility} can all be estimated using the arbitrary conditionals. For continuous target variables, the prediction can be estimated by drawing samples from $p(y \mid x_o)$, and we express their uncertainties using sample variances. We calculate the entropy terms in \eqref{eq:info_gain} with Monte Carlo estimations. The utility in \eqref{eq:utility} can be further simplified as
\begin{equation}\label{eq:cmi_continuous}
\begin{aligned}
    \mathcal{U}_i 
    &= \mathbb{E}_{p(y, x_i \mid x_o)} \log \frac{p(x_i, y \mid x_o)}{p(y \mid x_o)p(x_i \mid x_o)}\\ &= \mathbb{E}_{p(y,x_i \mid x_o)} \log \frac{p(y \mid x_i, x_o)}{p(y \mid x_o)}.
\end{aligned}
\end{equation}
We then perform a Monte Carlo estimation by sampling from $p(y, x_i \mid x_o)$. Note that $p(y \mid x_i, x_o)$ is evaluated on sampled $x_i$ rather than the true value, since we have not acquired its value yet.

For discrete target variables, we employ Bayes’ rule to make a prediction
\begin{equation}
\begin{aligned}
    P(y \mid x_o) &= \frac{p(x_o \mid y)P(y)}{\sum_{y'} p(x_o \mid y')P(y')}\\ &= \text{softmax}_y(\log p(x_o \mid y') + \log P(y')),
\end{aligned}
\end{equation}
and the uncertainty is expressed as the prediction probability. The information gain in \eqref{eq:info_gain} can be estimated analytically, since the entropy for a categorical distribution is analytically available. To estimate the utility, we further simplify \eqref{eq:cmi_continuous} to
\begin{equation}\label{eq:cmi_discrete}
\begin{aligned}
    \mathcal{U}_i &= \mathbb{E}_{p(x_i \mid x_o)P(y \mid x_i, x_o)} \log \frac{P(y \mid x_i,x_o)}{P(y \mid x_o)}\\
    &= \mathbb{E}_{p(x_i \mid x_o)} \KL[P(y \mid x_i, x_o) \| P(y \mid x_o)],
\end{aligned}
\end{equation}
where the KL divergence between two discrete distributions can be analytically computed. $x_i$ is sampled from $p(x_i \mid x_o)$ as before. We again use Monte Carlo estimation for the expectation.

Although the utility can be estimated accurately by \eqref{eq:cmi_continuous} and \eqref{eq:cmi_discrete}, it involves some overhead especially for long episodes, since we need to calculate them for each candidate feature at each acquisition step. Moreover, each Monte Carlo estimation may require multiple samples. To reduce the computation overhead, we utilize \eqref{eq:utility} and estimate the entropy terms with Gaussian approximations. That is, we approximate $p(x_i \mid x_o)$ and $p(x_i \mid y,x_o)$ as Gaussian distributions and entropies reduce to a function of the variance. We use sample variance as an approximation. We found that this Gaussian entropy approximation performs comparably while being much faster.

\subsubsection{Surrogate Model for AIR}
For unsupervised tasks, our goal is to represent the full set of features with an actively selected subset. Since the target is also $x$, we modify our surrogate model to capture arbitrary conditional distributions $p(x_u \mid x_o)$.
\begin{equation}\label{eq:cmi_air}
\resizebox{0.9\linewidth}{!}{$
    \mathcal{U}_i = H(x_i \mid x_o) - \mathbb{E}_{p(x \mid x_o)}H(x_i \mid x, x_o) = H(x_i \mid x_o).$}
\end{equation}
The last equality is due to the fact that $H(x_i \mid x) = 0$. We again use a Gaussian approximation to estimate the entropy. Therefore, the side information for AIR only contains imputed values and their variances of the unobserved features. Similar to the supervised case, we leverage the surrogate model to provide the intermediate rewards. Instead of using the information gain in \eqref{eq:info_gain}, we use the reduction of negative log likelihood per dimension, i.e.,
\begin{equation}\label{eq:bpd_gain}
\resizebox{0.85\linewidth}{!}{$
    r_m(s,i) = \frac{-\log p(x_u \mid x_o)}{|u|} - \gamma \frac{-\log p(x_{u \setminus i} \mid x_o, x_i)}{|u|-1},$}
\end{equation}
since \eqref{eq:info_gain} involves estimating the entropy for potentially high dimensional distributions, which itself is an open problem \citep{kybic2007high}. We show in appendix \ref{sec:invariance} that the optimal policy is invariant under this form of intermediate rewards.
The final reward $-\mathcal{L}(\hat{x},x)$ is calculated as the negative MSE of unobserved features
$-\mathcal{L}(\hat{x},x) = -\Vert \hat{x}_u - x_u \Vert_2^2$.

\subsection{AFA for Time Series}
We also apply our \method~framework on time series data. For example, consider a scenario where sensors are deployed in the field with limited resources. We would like the sensors to decide when to put themselves online to collect data. The goal is to make as few acquisitions as possible while still making an accurate prediction. In contrast to ordinary vector data, the acquired features must follow a chronological order, i.e., the newly acquired feature $i$ must occur after all elements of $o$ (since we may not go back in time to turn on sensors). In this case, it is detrimental to acquire a feature that occurs very late in an early acquisition step, since we will lose the opportunity to observe features ahead of it. The chronological constraint in action space removes all the features behind the acquired features from the candidate set. For example, after acquiring feature $t$, features $\{1,\ldots,t\}$ are no longer considered as candidates for the next acquisition. 

\subsection{Implementation}
We implement our \method~framework using the Proximal Policy Optimization (PPO) algorithm \citep{schulman2017proximal}. The policy network takes in a set of observed features and a set of auxiliary information from the surrogate model, extracts a set embedding from them using the set transformer \citep{lee2019set}, and outputs the actions. The critic network that estimates the value function shares the same set embedding as the policy network. To help learn useful representations, we also use the same set embedding as inputs for the prediction model $f_\theta$. Arbitrary conditionals are estimated based on \citep{li2019flow}.

To reflect the fact that acquiring the same feature repeatedly is redundant, we manually remove those acquired features from the candidate set. For time-series data, the acquired features must follow the chronological order since we cannot go back in time to acquire another feature, therefore we need to remove all the features behind the acquired features from the candidate set. Similar spatial constraints can also be applied for spatial data. To satisfy those constraints, we manually set the probabilities of the invalid actions to zeros.

\begin{figure*}
    \centering
    \begin{minipage}{0.49\linewidth}
    \includegraphics[width=\textwidth]{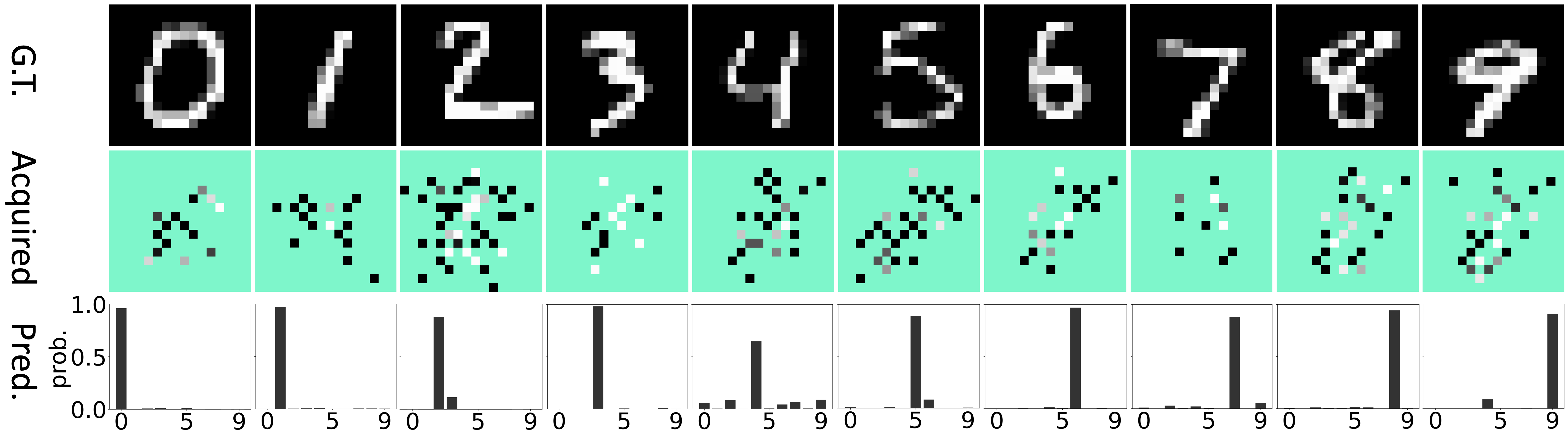}
    \vspace{-18pt}
    \caption{Example of acquired features and prediction. The green masks indicate the unobserved features.}
    \label{fig:mnist_cls}
    \end{minipage}
    \hspace{10pt}
    \begin{minipage}{0.46\linewidth}
    \includegraphics[width=\textwidth]{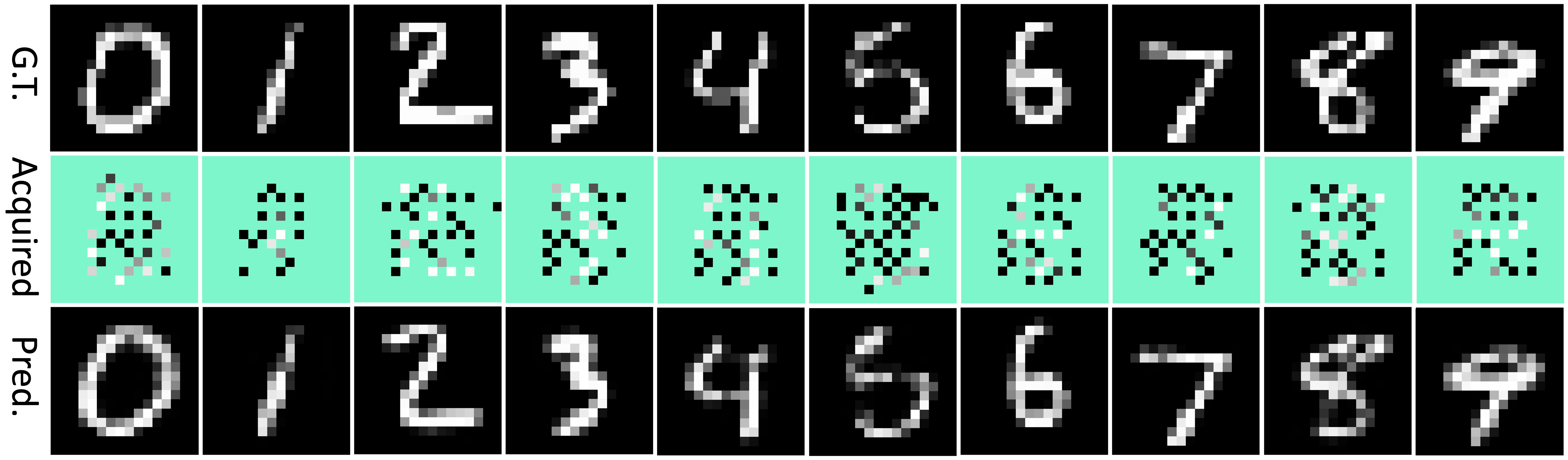}
    \vspace{-18pt}
    \caption{Example of acquired features and inpaintings. The green masks indicate the unobserved features.}
    \label{fig:mnist_air}
    \end{minipage}
    \vspace{-5pt}
\end{figure*}

\begin{figure*}
    \centering
    \begin{minipage}{0.49\textwidth}
    \subfigure{\includegraphics[width=0.5\textwidth]{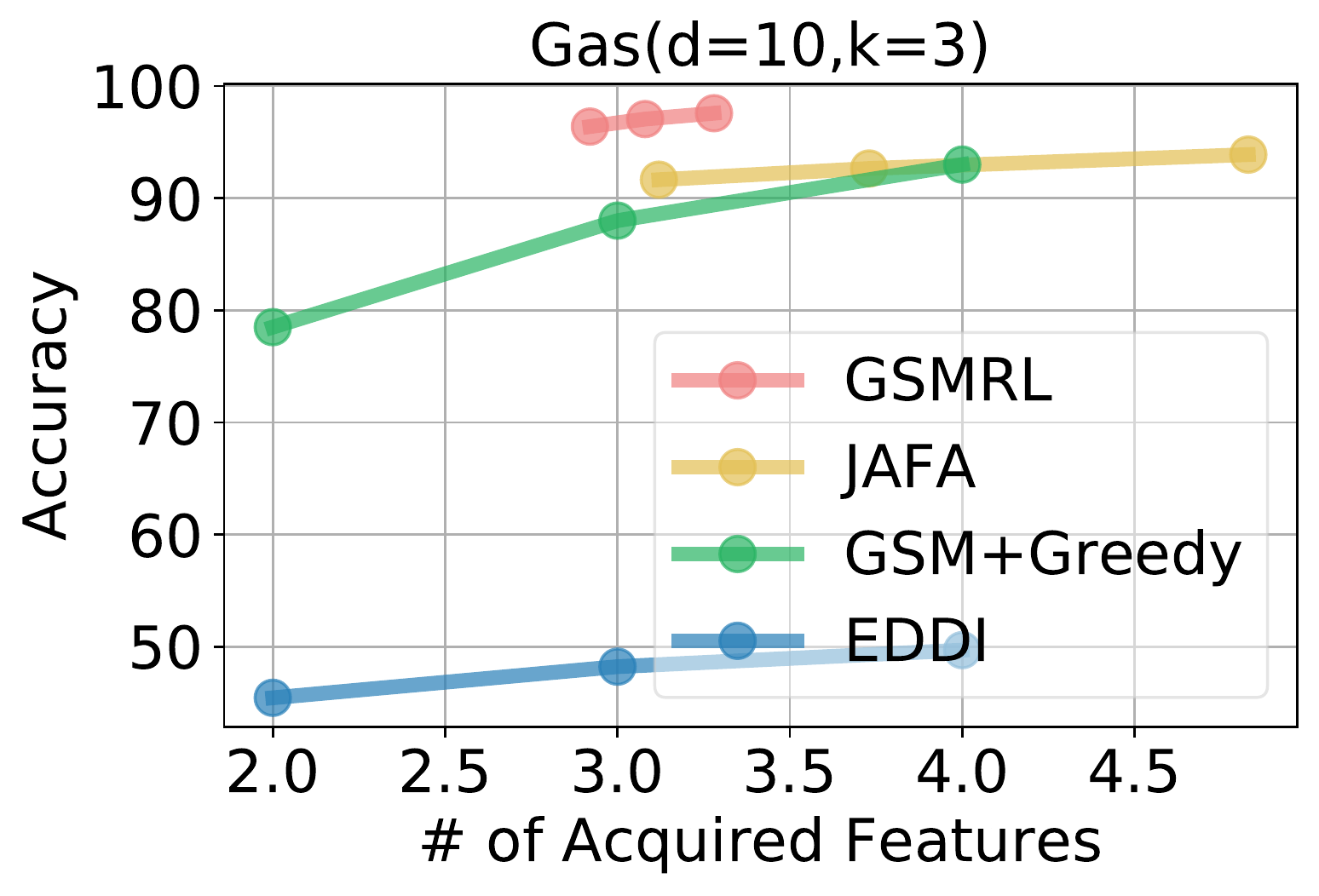}}
    \subfigure{\includegraphics[width=0.48\textwidth]{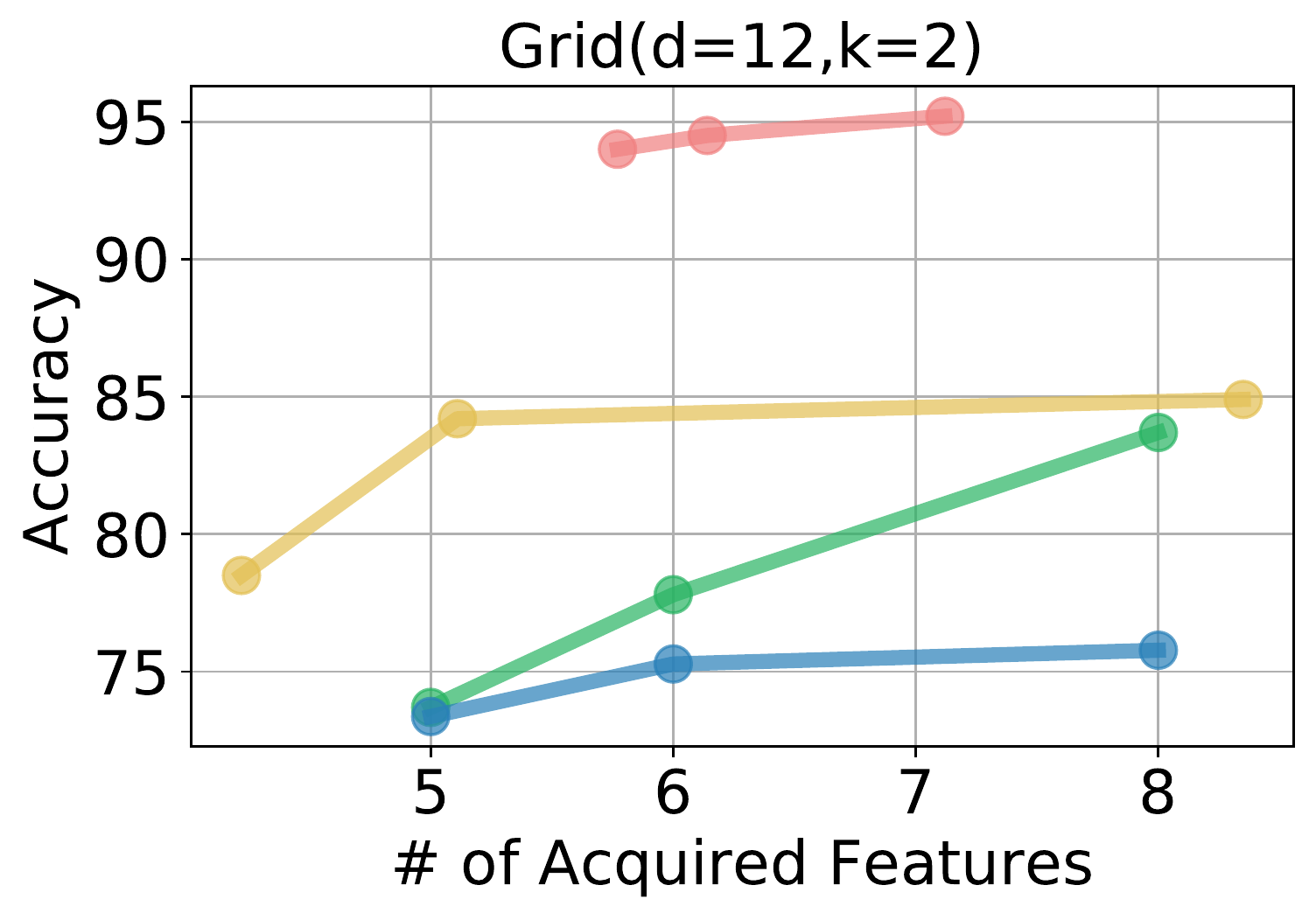}}
    \vspace{-10pt}
    \caption{Test accuracy on UCI datasets.}
    \label{fig:uci_cls_acc}
    \end{minipage}
    \begin{minipage}{0.49\textwidth}
    \subfigure{\includegraphics[width=0.49\textwidth]{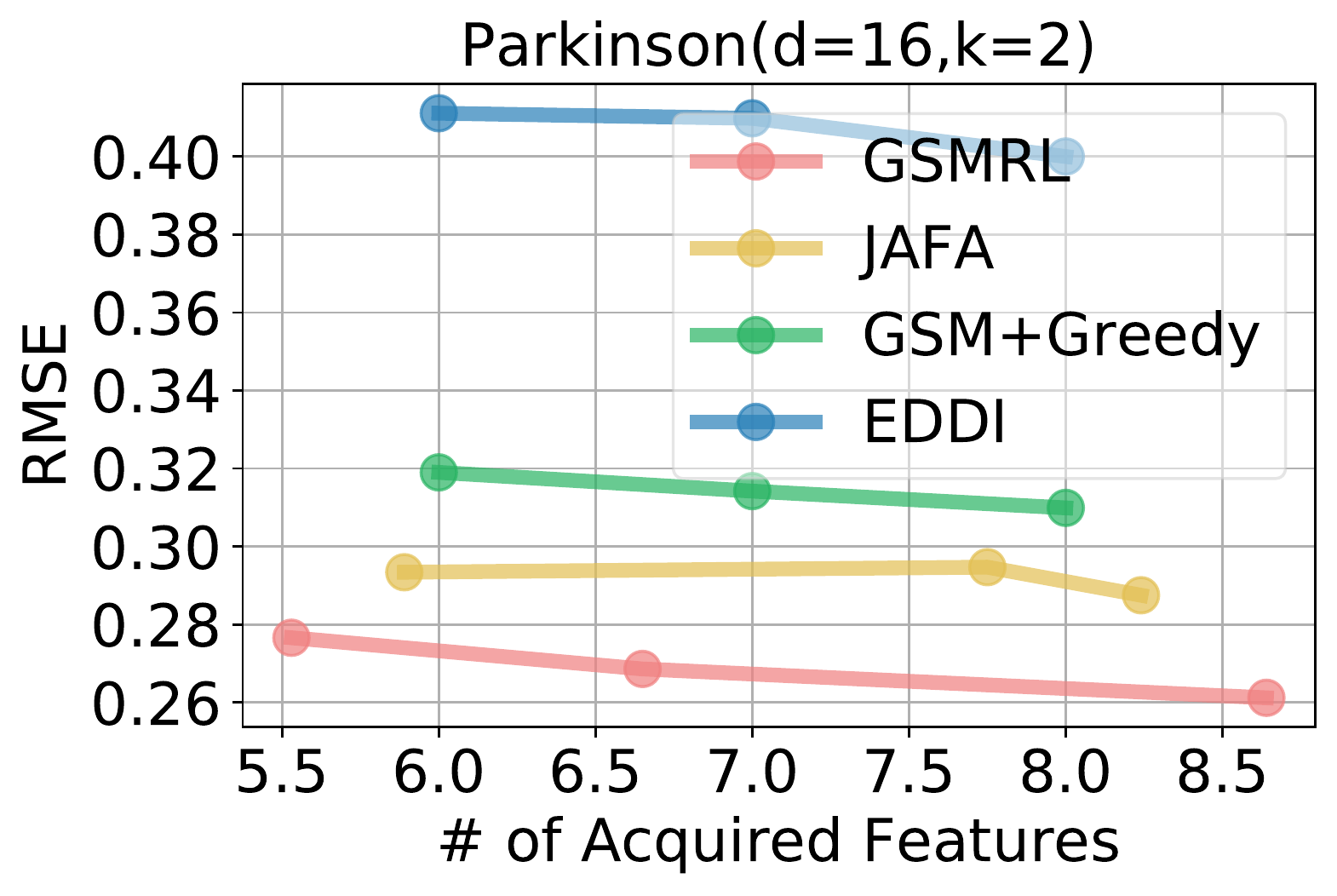}}
    \subfigure{\includegraphics[width=0.49\textwidth]{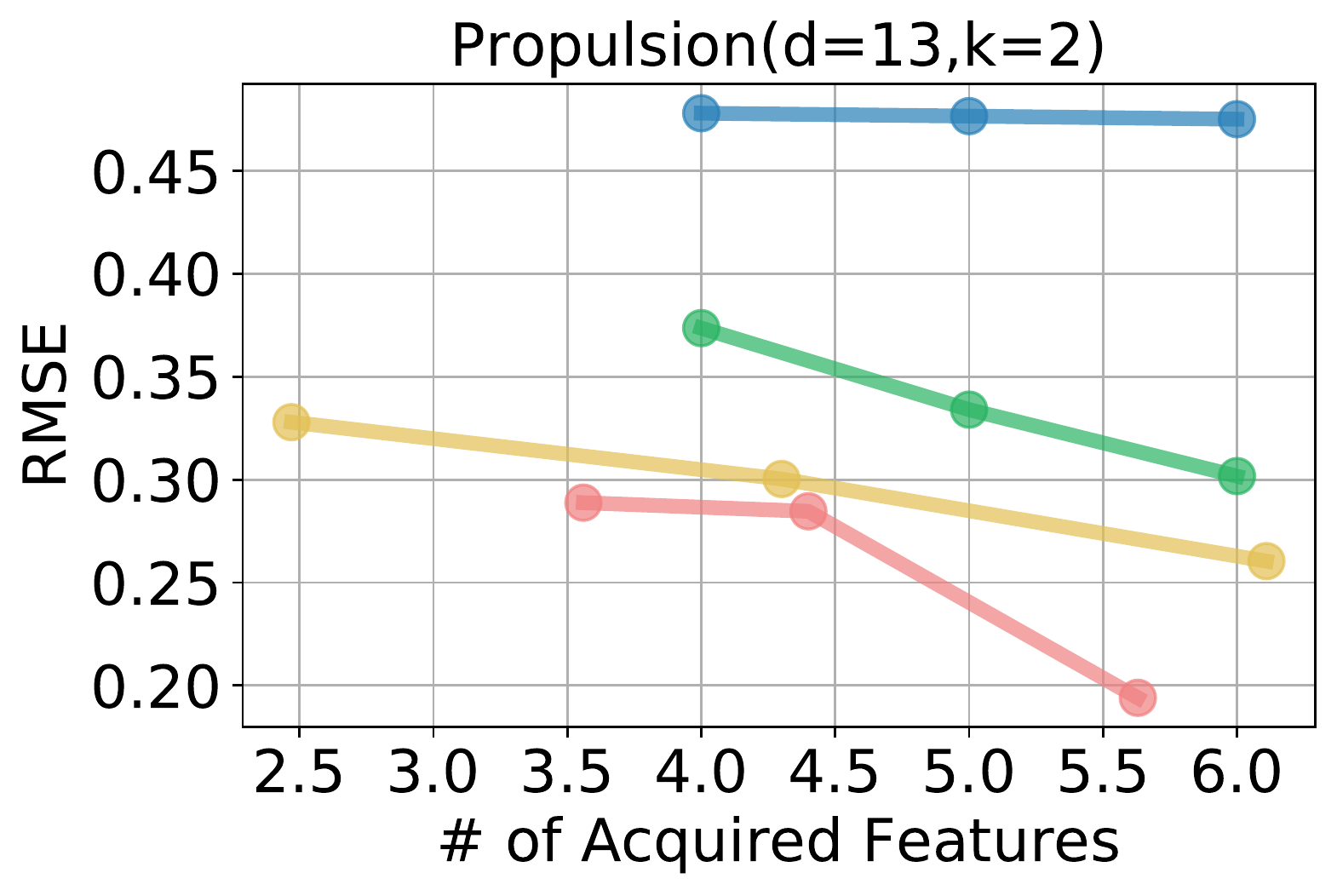}}
    \vspace{-10pt}
    \caption{Test RMSE on UCI datasets.}
    \label{fig:uci_reg_rmse}
    \end{minipage}
    \vspace{-8pt}
\end{figure*}

\section{Related Works}

\textbf{Active Learning}
Active learning \citep{fu2013survey,konyushkova2017learning,yoo2019learning} is a related approach in ML to gather more information when a learner can query an oracle for the true label, $y$, of a complete feature vector $x \in \mathbb{R}^d$ to build a better estimator. However, our methods consider queries to the environment for the feature value corresponding to an unobserved feature dimension, $i$, in order to provide a better prediction on the current instance. Thus, while the active learning paradigm queries an oracle \emph{during training} to build a classifier with complete features, our paradigm queries the environment \emph{at evaluation} to obtain missing features of a current instance to help its current assessment.

\textbf{Feature Selection}
Feature selection \citep{miao2016survey,li2017feature,cai2018feature}, ascertains a static subset of important features to eliminate redundancies, which can help reduce computation and improve generalization. Feature selection methods choose a \emph{fixed} subset of features $s \subseteq \{1,\dots, d\}$, and always predict $y$ using this same subset of feature values, $x_s$. In contrast, our model considers a \emph{dynamic} subset of features that is sequentially chosen and personalized on an instance-by-instance basis to increase useful information. 

\textbf{Active Feature Acquisition}
Instead of predicting the target passively using collected features, previous works have explored actively acquiring features in the cost-sensitive setting. Active perception is a relevant sub-field where a robot with a mounted camera is planning by selecting the best next view \citep{bajcsy1988active,aloimonos1988active, cheng2018reinforcement,jayaraman2018learning}. In this work we consider general features, and take images as one of many data sources. For general data, \citet{ling2004decision}, \citet{chai2004test} and \citet{nan2014fast} propose decision tree, naive Bayes and maximum margin based classifiers, respectively, to jointly minimize the misclassification cost and feature acquisition cost. \citet{ma2018eddi} and \citet{gong2019icebreaker} acquire features greedily using mutual information as the estimated utility. \citet{zubek2004pruning} formulate the AFA problem as a MDP and fit a transition model using complete data, then they use the AO* heuristic search algorithm to find an optimal policy. \citet{ruckstiess2011sequential} formulate the problem as a partially observable MDP and solve it using Fitted Q-Iteration. \citet{he2012imitation} and \citet{he2016active} instead employ a imitation learning approach guided by a greedy reference policy. \citet{shim2018joint} utilize Deep Q-Learning and jointly learn a policy and a classifier. The classifier is treated as an environment that calculates the classification loss as the reward. ODIN \citep{zannone2019odin} presents an approach to learn a policy and a prediction model using augmented data with a Partial VAE \citep{ma2018eddi}. In contrast, GSMRL uses a surrogate model, which estimates both the state transitions and the prediction in a unified model, to directly provide intermediate rewards and auxiliary information to an agent.


\textbf{Model-based and Model-free RL}
Reinforcement learning can be roughly grouped into model-based methods and model-free methods depending on whether they use a transition model \citep{li2017deep}. Model-based methods are more data efficient but could suffer from significant bias if the dynamics are misspecified. On the contrary, model-free methods can handle arbitrary dynamic system but typically requires substantially more data samples. There have been works that combine model-free and model-based methods to compensate with each other. The usage of the model includes generating synthetic samples to learn a policy \citep{gu2016continuous}, back-propagating the reward to the policy along a trajectory \citep{heess2015learning}, and planning \citep{chebotar2017combining,pong2018temporal}. In this work, we rely on the model to provide intermediate rewards and side information. We compare this strategy to other model-based approaches in section \ref{sec:ablations}.

\section{Experiments}
In this section, we evaluate our method on several benchmark environments built upon the UCI repository \citep{Dua:2019} and MNIST dataset \citep{lecun1998mnist}. We compare our method to another RL based approach, JAFA \citep{shim2018joint}, which jointly trains an agent and a classifier. We also compare to a greedy policy EDDI \citep{ma2018eddi} that estimates the utility for each candidate feature using a VAE based model and selects one feature with the highest utility at each acquisition step. As a baseline, we also acquire features greedily using our surrogate model that estimates the utility following \eqref{eq:cmi_continuous}, \eqref{eq:cmi_discrete} and \eqref{eq:cmi_air}. We use a fixed cost for each feature and report multiple results with different $\alpha$ in \eqref{eq:goal} to control the trade-off between task performance and acquisition cost. We cross validate the best architecture and hyperparameters for baselines. Architectural details, hyperparameters and sensitivity analysis are provided in the Appendix. In this work, we conduct the largest scale AFA study to date.  
Previous works have typically considered smaller datasets (both in terms of the number of features and the number of instances).
We instead consider a broad range of datasets with more instances and higher dimensionality. In terms of comparisons, previous works often compare to naively simple baselines, such as a random acquisition order. In this work, we compare our GSMRL to the state-of-the-art models with both greedy policy and non-greedy RL policy.

\begin{figure*}
    \centering
    \begin{minipage}{0.49\textwidth}
    \subfigure{\includegraphics[width=0.49\textwidth]{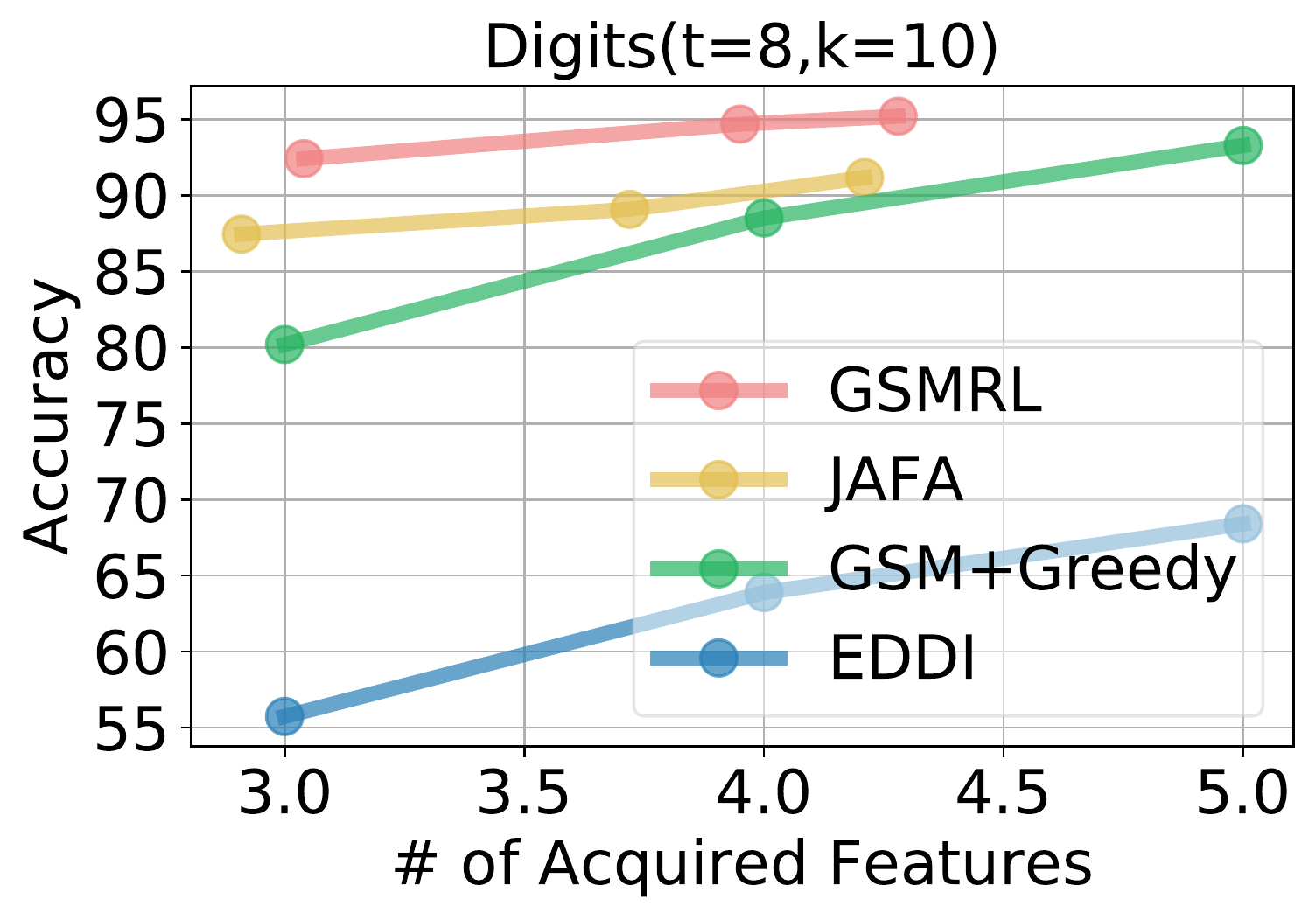}}
    \subfigure{\includegraphics[width=0.49\textwidth]{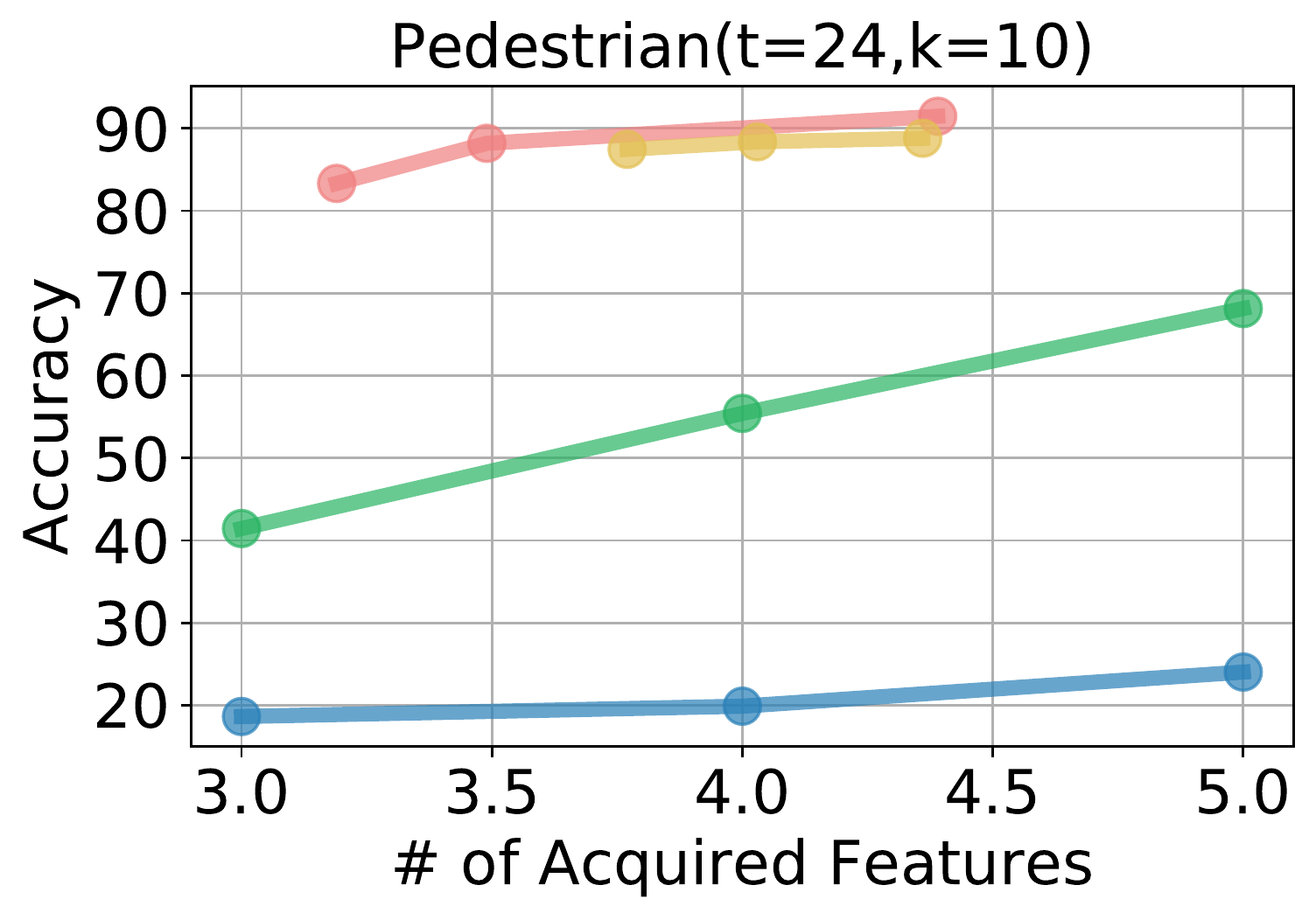}}
    \vspace{-10pt}
    \caption{Classification on time series.}
    \label{fig:uci_ts_acc}
    \end{minipage}
    \begin{minipage}{0.49\textwidth}
    \subfigure{\includegraphics[width=0.49\textwidth]{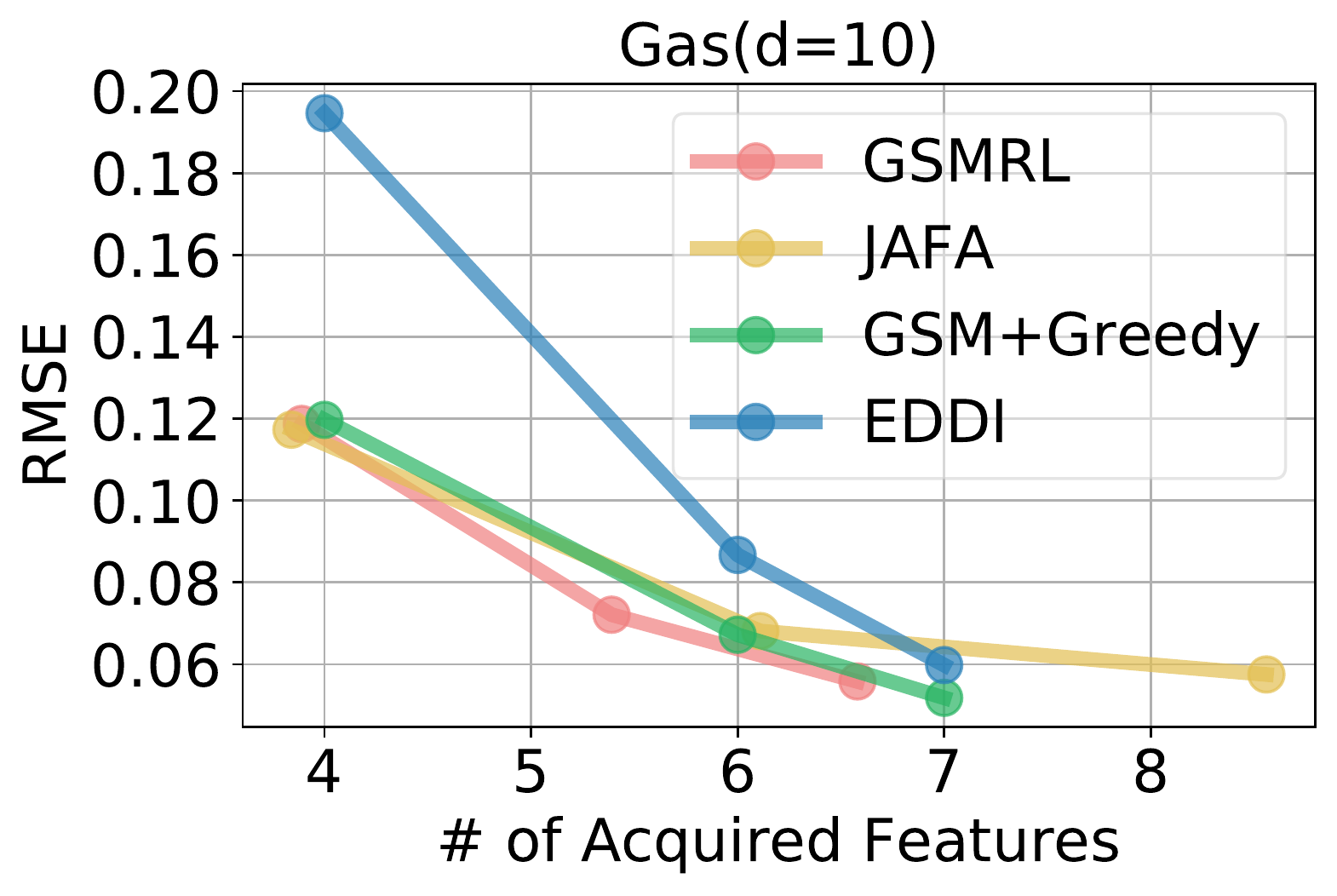}}
    \subfigure{\includegraphics[width=0.49\textwidth]{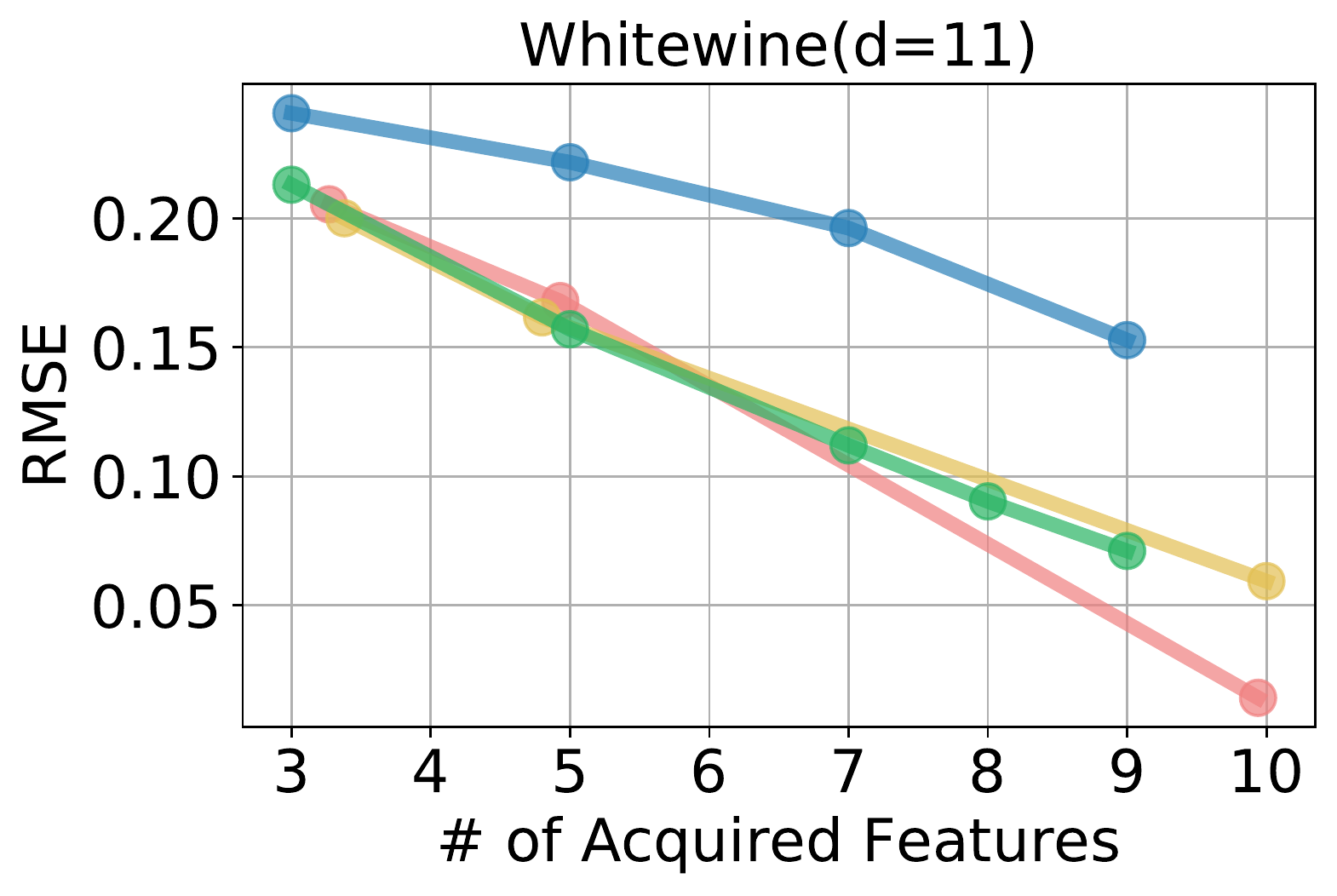}}
    \vspace{-10pt}
    \caption{RMSE for unsupervised tasks.}
    \label{fig:uci_air_rmse}
    \end{minipage}
    \vspace{-8pt}
\end{figure*}

\textbf{Classification}
We first perform classification on the MNIST dataset. We downsample the original images to $16\times16$ to reduce the action space to accommodate baselines such as EDDI that have trouble scaling (see Sec.~\ref{sec:results} in the appendix for details on full MNIST).
\begin{wrapfigure}{r}{0.45\linewidth}
    \centering
    \vspace{-5pt}
    \includegraphics[width=\linewidth]{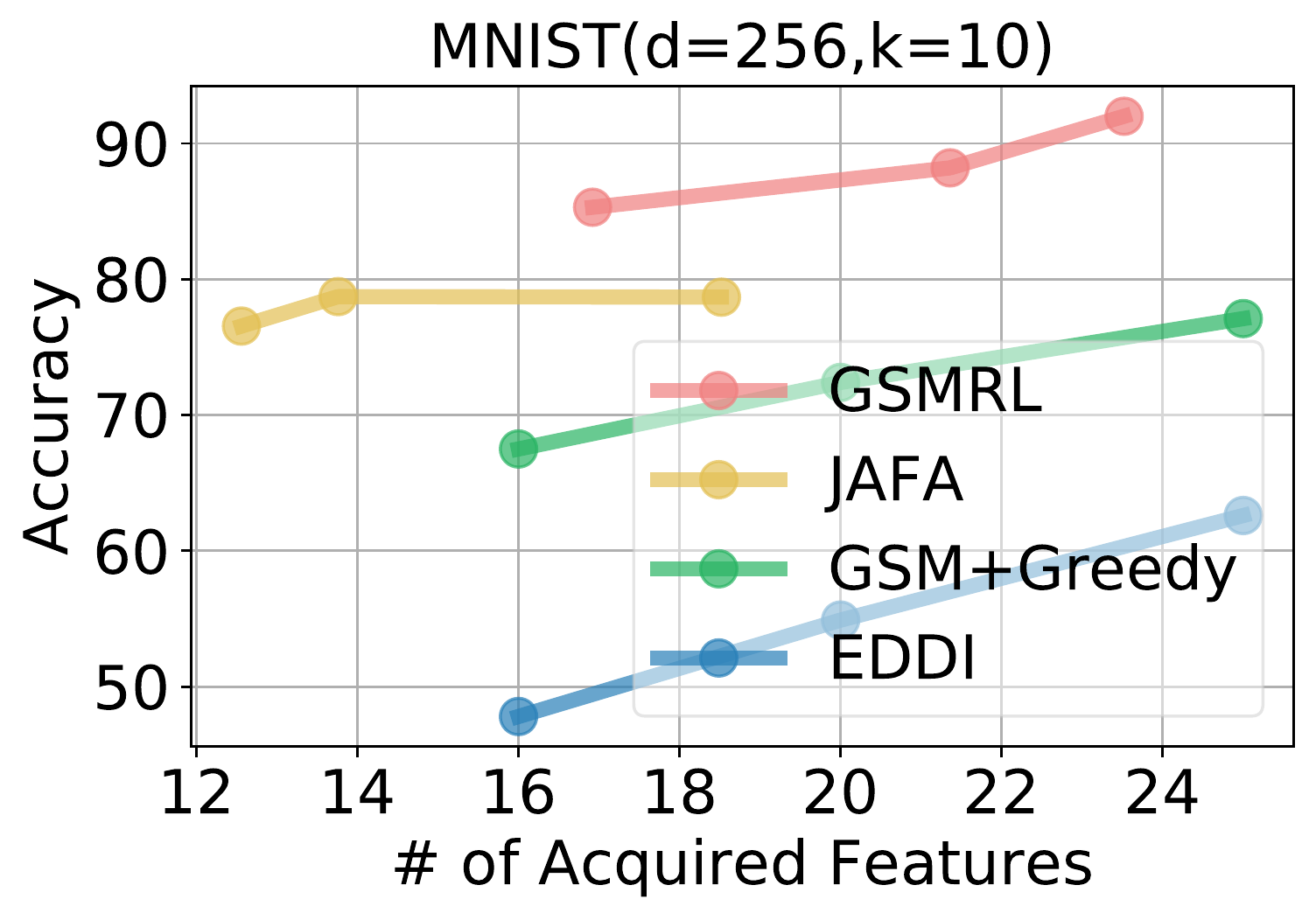}
    \vspace{-20pt}
    \caption{Test accuracy for AFA on MNIST.}
    \label{fig:mnist_cls_acc}
    \vspace{-5pt}
\end{wrapfigure}
Fig.~\ref{fig:mnist_cls} illustrates several examples of the acquired features and their prediction probability for different images. We can see that our model acquires a different subset of features for different images. Notice the checkerboard patterns of the acquired features, which indicates our model is able to exploit the spatial correlation of the data. Fig.~\ref{fig:mnist_cls_acq} shows the acquisition process and the prediction probability along the acquisition. We can see the prediction become certain after acquiring only a small subset of features. The test accuracy in Fig.~\ref{fig:mnist_cls_acc} demonstrates the superiority of our method over other baselines. It typically achieves higher accuracy with a lower acquisition cost. It is worth noting that our surrogate model with a greedy acquisition policy outperforms EDDI. We believe the improvement is due to the better distribution modeling ability of our  surrogate model so that the utility and the prediction are more accurately estimated. We also perform classification using several UCI datasets. The test accuracy is presented in Fig.~\ref{fig:uci_cls_acc}. Again, our method outperforms baselines under the same acquisition budget.

\textbf{Regression}
We also conduct experiments for regression tasks using several UCI datasets. We report the root mean squared error (RMSE) of the target variable in Fig.~\ref{fig:uci_reg_rmse}. Similar to the classification task, our model outperforms baselines with a lower acquisition cost.

\textbf{Time Series}
To evaluate the performance with constraints in action space, we classify over time series data where the acquired features must follow chronological ordering. The datasets are from the UEA \& UCR time series classification repository \citep{bagnall2017great}. For \method~and JAFA, we clip the probability of invalid actions to zero; for the greedy method, we use a prior to bias the selection towards earlier time points. Please refer to Appendix \ref{sec:time_series} for details. Fig.~\ref{fig:uci_ts_acc} shows the accuracy with different numbers of acquired features. Our method achieves high accuracy by collecting a small subset of the features. 

\textbf{Medical Diagnosis}
We evaluate the AFA performance for medical diagnosis. We use the Physionet challenge
\begin{wrapfigure}{r}{0.45\linewidth}
    \centering
    \vspace{-5pt}
    \includegraphics[width=\linewidth]{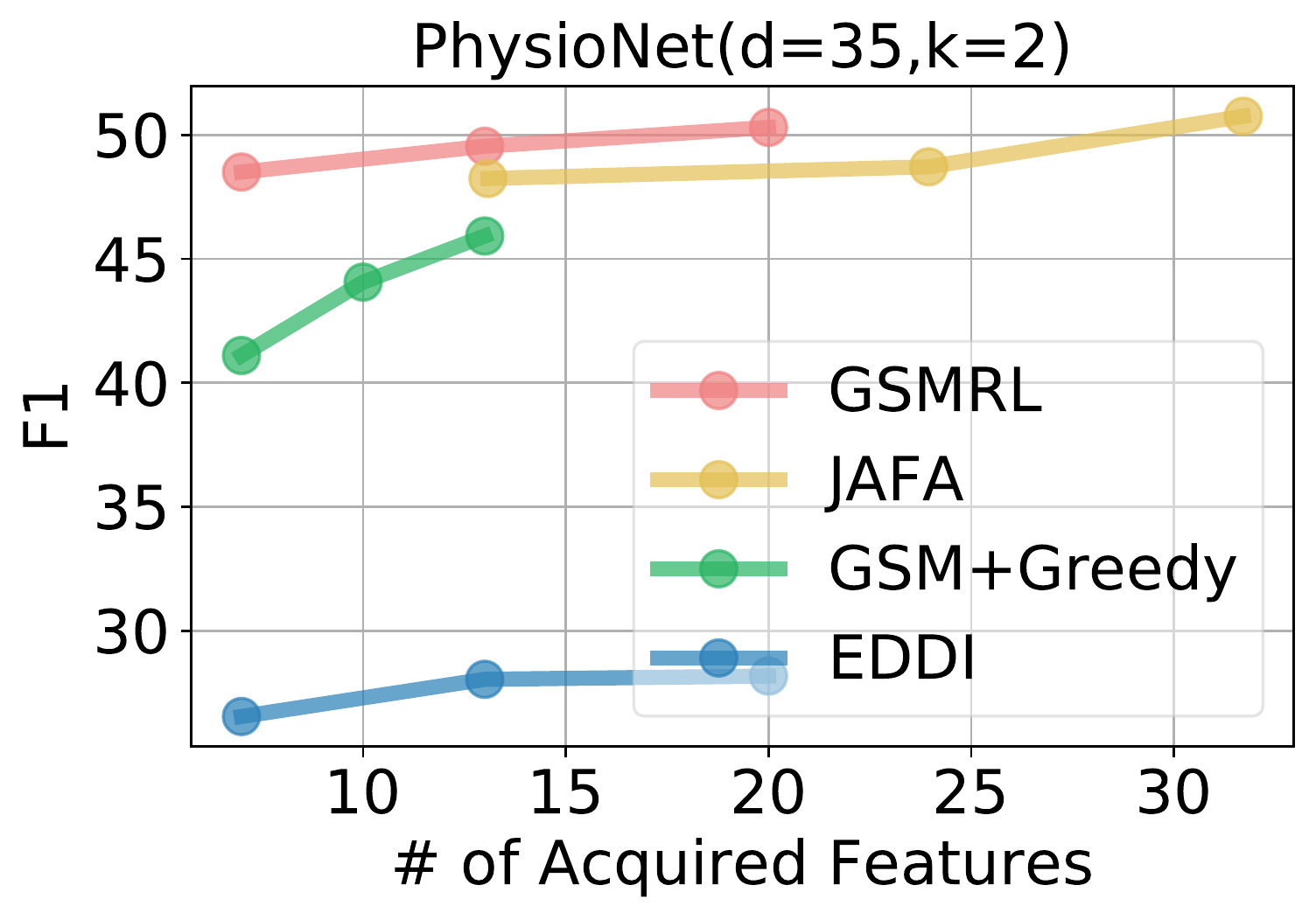}
    \vspace{-20pt}
    \caption{F1 for in-hospital mortality on Physionet.}
    \label{fig:physionet12}
    \vspace{-5pt}
\end{wrapfigure}
2012 dataset \cite{goldberger2000physiobank} to predict the in-hospital mortality. Since the classes are heavily imbalanced, we use weighted cross entropy as training loss and the final rewards. For evaluation, we report the F1 scores in Fig.~\ref{fig:physionet12}. Compared to baselines, our model achieves higher F1 with lower acquisition cost.

\textbf{Unsupervised}
Next, we evaluate our method on unsupervised tasks where features are actively acquired to impute the unobserved features. 
\begin{wrapfigure}{r}{0.45\linewidth}
    \centering
    \vspace{-5pt}
    \includegraphics[width=\linewidth]{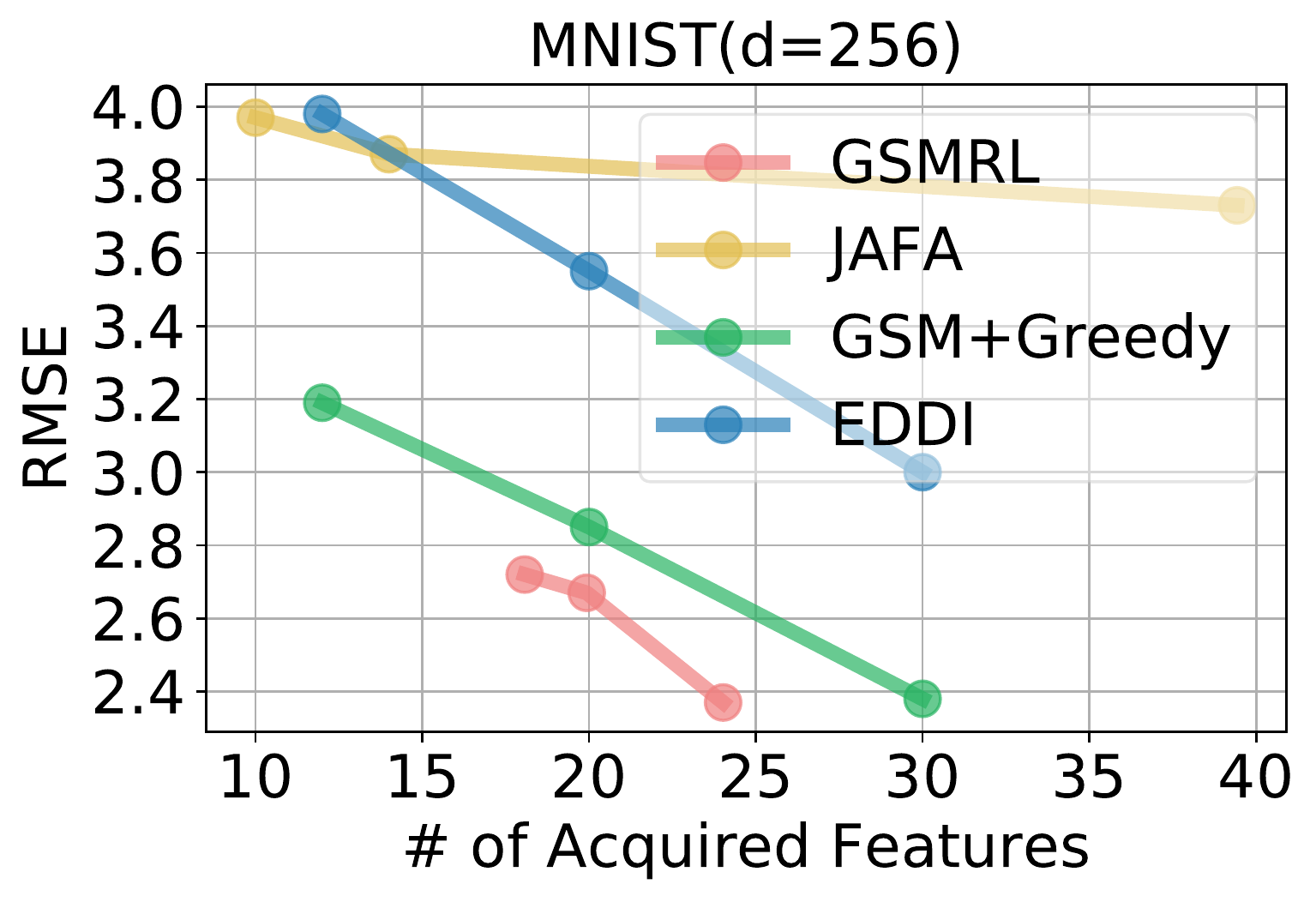}
    \vspace{-20pt}
    \caption{RMSE of $x_u$ for AIR on MNIST.}
    \label{fig:mnist_air_rmse}
    \vspace{-5pt}
\end{wrapfigure}
We use negative MSE as the reward for \method~and JAFA. The greedy policy calculates the utility following \eqref{eq:cmi_air}. For low dimensional UCI datasets, our method is comparable to baselines as shown in Fig.~\ref{fig:uci_air_rmse}; but for the high dimensional case, as shown in Fig.~\ref{fig:mnist_air_rmse}, our method is doing better. Note JAFA is worse than the greedy policy for MNIST. We found it hard to train the policy and the reconstruction model jointly without the help of the surrogate model in this case. See Fig.~\ref{fig:mnist_air_acq} for an example of the acquisition process.

\section{Ablations}\label{sec:ablations}
We now conduct a series of ablation studies to explore the capabilities of our GSMRL model.

\textbf{Model-based Alternatives}
Our GSMRL model combines model-based and model-free approach into a holistic framework by providing the agent with auxiliary information and intermediate rewards. Here, we study different ways of utilizing the dynamics model. As in ODIN \cite{zannone2019odin}, we utilize class conditioned generative models to generate synthetic trajectories. 
\begin{wrapfigure}{r}{0.45\linewidth}
    \centering
    \vspace{-5pt}
    \includegraphics[width=\linewidth]{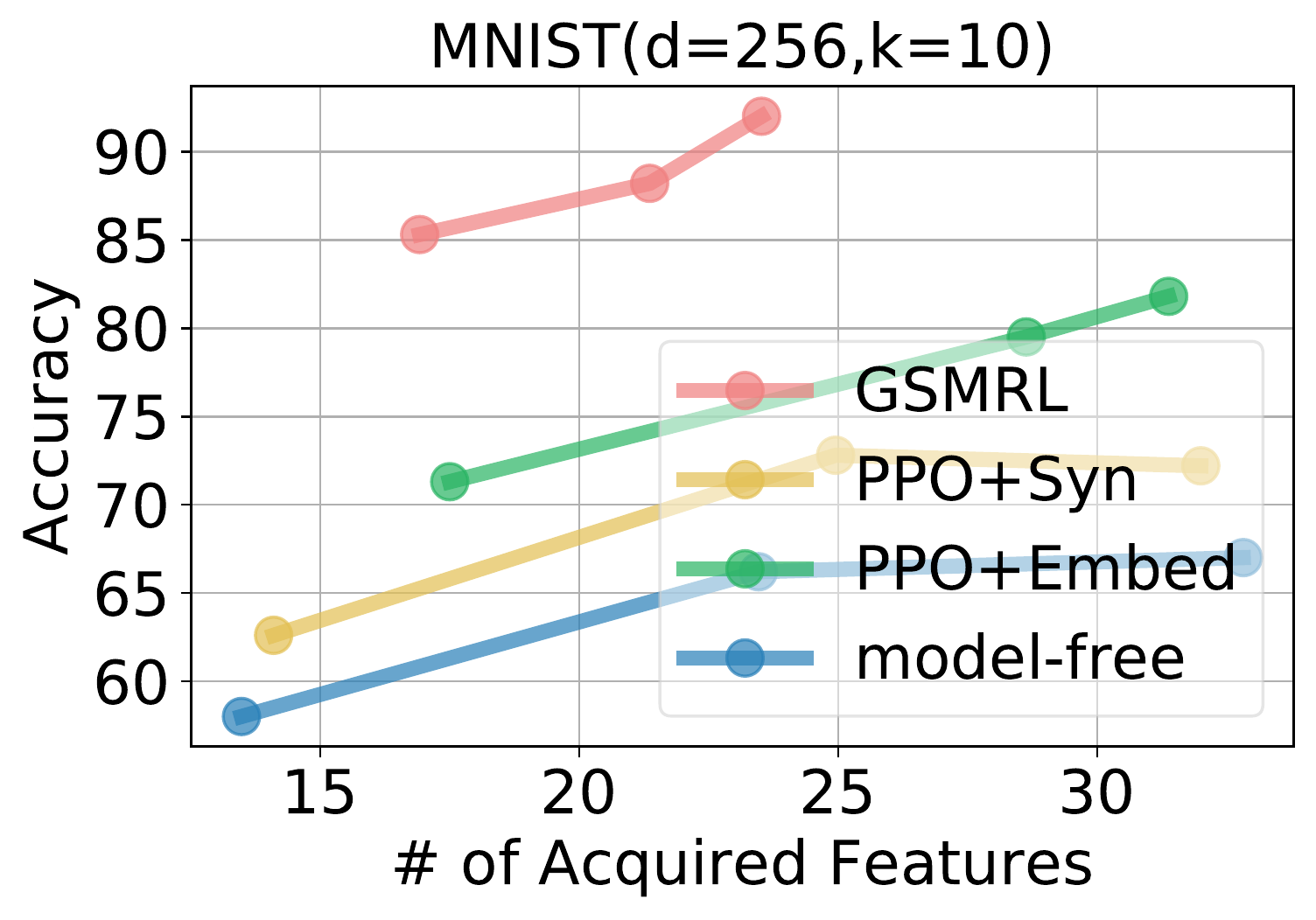}
    \vspace{-20pt}
    \caption{Other model-based approaches.}
    \label{fig:model-based}
    \vspace{-5pt}
\end{wrapfigure}
The agent is then trained with both real and synthetic data (PPO+Syn). Another way of using the model is to extract a semantic embedding from the observations \cite{kumar2018consistent}. We use a pretrained EDDI to embed the current observed features into a 100-dimensional feature vector. An agent then takes the embedding as input and predicts the next acquisition (PPO+Embed). Figure \ref{fig:model-based} compares our method with these alternatives. We also present the results from a model-free approach as a baseline. We see our GSMRL outperforms other model-based approaches by a large margin.

\textbf{Surrogate Models}
Our method relies on the surrogate model to provide intermediate rewards and auxiliary information. To better understand the contributions each component does to the overall framework, we conduct ablation studies using the MNIST dataset. We gradually drop one component from the full model and report the results in Fig.~\ref{fig:ablation}. The `Full Model' uses both intermediate rewards and auxiliary information. 
\begin{wrapfigure}{r}{0.45\linewidth}
    \vspace{-5pt}
    \begin{minipage}{\linewidth}
    \includegraphics[width=\linewidth]{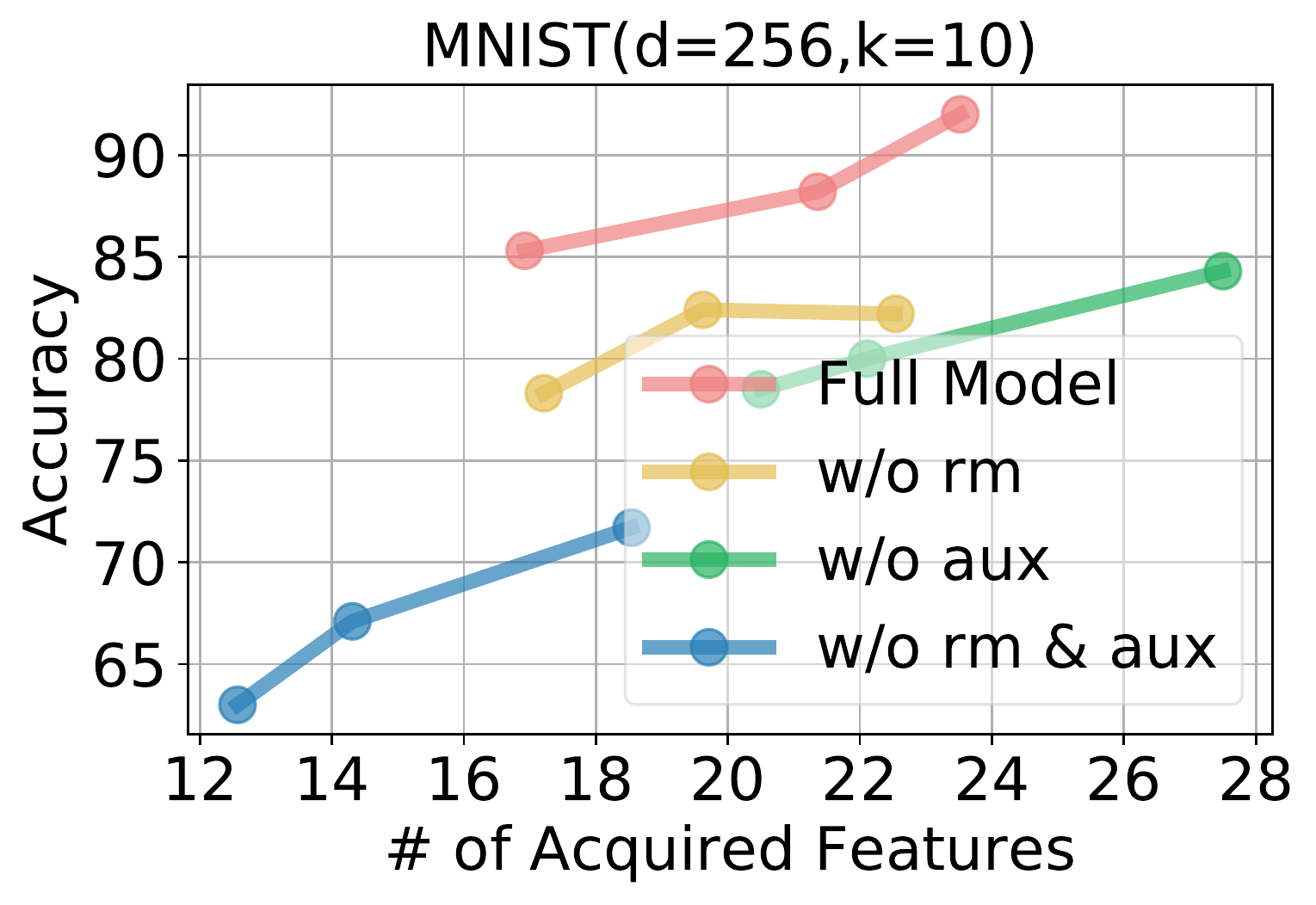}
    \vspace{-20pt}
    \caption{Ablations}
    \label{fig:ablation}
    \vspace{7pt}
    \end{minipage}
    \begin{minipage}{\linewidth}
    \includegraphics[width=\linewidth]{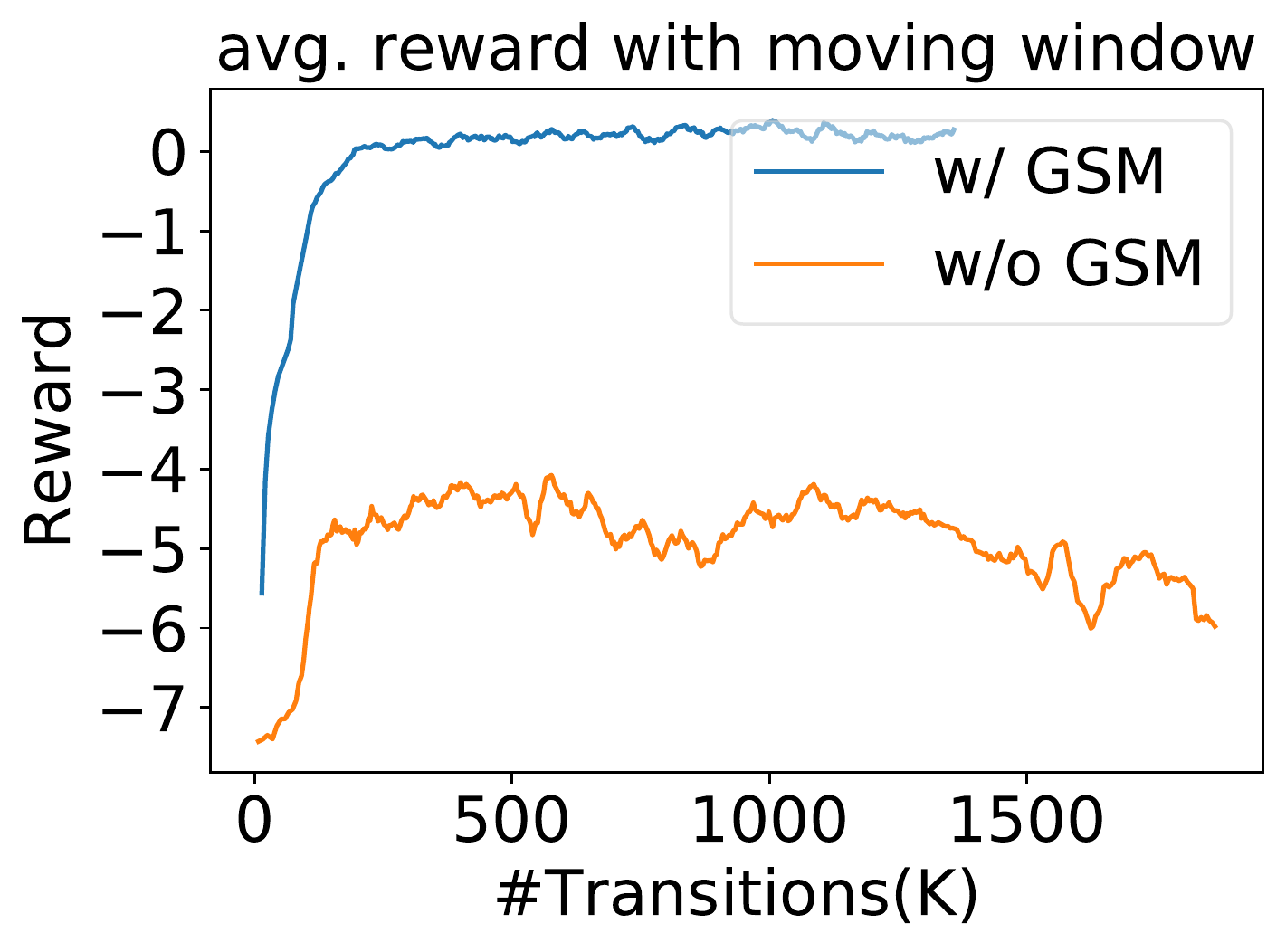}
    \vspace{-20pt}
    \caption{Rewards}
    \label{fig:rewards}
    \end{minipage}
    \vspace{-7pt}
\end{wrapfigure}
We then drop the intermediate rewards and denote it as `w/o rm'. The model without auxiliary information is denoted as `w/o aux'. We further drop both components and denote it as `w/o rm \& aux'. From Fig.~\ref{fig:ablation}, we see these two components contribute significantly to the final results. We also compare models with and without the surrogate model. For models without a surrogate model, we train a classifier jointly with the agent as in JAFA. We plot the smoothed rewards using moving window average during training in Fig.~\ref{fig:rewards}. The agent with a surrogate model not only produces higher and smoother rewards but also converges faster.

\textbf{Dynamic vs. Static Acquisition}
Our GSMRL acquires features following a dynamic order where it eventually acquires different features for different instances.
A dynamic acquisition policy should perform better than a static one (i.e., the same set of features are acquired for each instance), since the dynamic policy allows the acquisition to be specifically adapted to the corresponding instance. To verify this is actually the case, we compare the dynamic and static acquisition under a greedy policy for MNIST classification. 
\begin{wrapfigure}{r}{0.45\linewidth}
\vspace{-5pt}
\includegraphics[width=\linewidth]{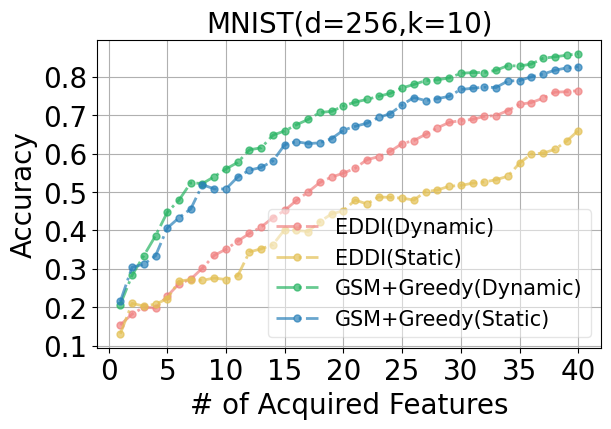}
\vspace{-20pt}
\caption{Compare dynamic and static acquisition strategy using greedy policies.}
\label{fig:mnist_dyn_stc}
\vspace{-5pt}
\end{wrapfigure}
Similar to the dynamic greedy policy, the static acquisition policy acquires the feature with maximum utility at each step, but the utility is averaged over the whole testing set, therefore the same acquisition order is adopted for the whole testing set. Figure \ref{fig:mnist_dyn_stc} shows the classification accuracy for both EDDI and GSM under a greedy acquisition policy. We can see the dynamic policy is always better than the corresponding static one. Furthermore, our GSM with a static acquisition can already outperform dynamic EDDI.

\textbf{Greedy vs. Non-greedy Acquisition}
Our GSMRL will terminate the acquisition process if the agent deems the current acquisition achieves the optimal trade-off between the prediction performance and the acquisition cost. 
\begin{wrapfigure}{r}{0.45\linewidth}
\vspace{-7pt}
\includegraphics[width=\linewidth]{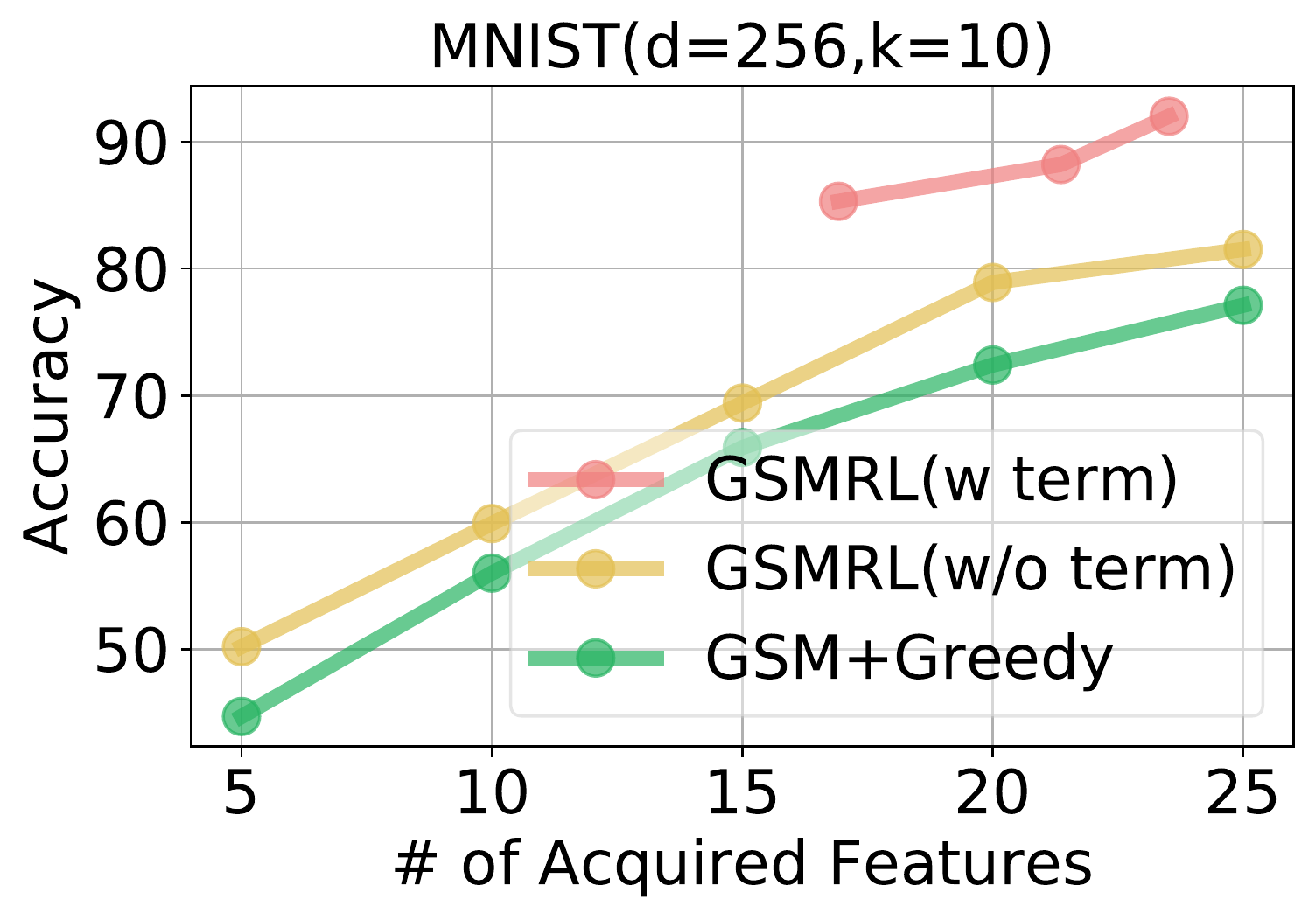}
\vspace{-20pt}
\caption{Acquisition with and without termination.}
\label{fig:mnist_fixed_budget}
\vspace{-5pt}
\end{wrapfigure}
To evaluate how much the termination action affects the performance and to directly compare with the greedy policies under the same acquisition budget, we conduct an ablation study that removes the termination action  
and gives the agent a hard acquisition budget 
(i.e., forcing the agent to predict after some number of acquisitions).
We can see (Fig.~\ref{fig:mnist_fixed_budget}) GSMRL outperforms the greedy policy under all budgets.
Moreover, we see that the agent is able to correctly assess whether or not more acquisitions are useful, since it obtains better performance when it dictates when to predict with the termination action.

\vspace{-2pt}
\section{Conclusion}
In this work, we reformulate the dynamics of the AFA MDP as a generative modeling task among features.
We leverage a generative surrogate model to capture the state transitions across arbitrary feature subsets. The surrogate model also provides auxiliary information and intermediate rewards to assist the agent. Our \method model essentially combines model-based and model-free approaches. We conduct a large scale study to evaluate our model on both supervised and unsupervised AFA problems. Our model achieves state-of-the-art performance on both problems. In future work, we will explore AFA in spatial-temporal setting with continuously indexed features.

\bibliography{main}
\bibliographystyle{icml2021}

\clearpage
\appendix

\renewcommand\thefigure{\thesection.\arabic{figure}}
\renewcommand\thetable{\thesection.\arabic{table}}
\renewcommand{\theequation}{\thesection.\arabic{equation}}

\setcounter{figure}{0} 
\setcounter{subfigure}{0}
\setcounter{equation}{0}

\section{Policy Invariance under Intermediate Rewards}\label{sec:invariance}
Assume the original Markov Decision Process (MDP) without the intermediate rewards is defined as $M = (S,A,T,\gamma,R)$, where $S$ and $A$ are state and action spaces, $T$ is the state transition probabilities, $\gamma$ is the discount factor, and $R$ is the rewards. When we introduce the intermediate rewards $R_m$, the MDP is modified to $M'=(S,A,T,\gamma,R')$, where $R'=R+R_m$. The following theory provides a sufficient and necessary condition for the modified MDP $M'$ to achieve the same optimal policy as the original MDP $M$.

\begin{theorem}
The modified MDP $M'=(S,A,T,\gamma,R+F)$ with any shaping reward function $F$ is guaranteed to be consistent with the optimal policy of the original MDP $M=(S,A,T,\gamma,R)$ if the shaping function $F$ have the following form
\begin{equation}
    F(s,a,s') = \gamma \Phi(s') - \Phi(s),
\end{equation}
where $\Phi: S \rightarrow \mathbb{R}$ is a potential function evaluated on states. For infinite-state case (i.e., the state space is an infinite set) the potential function is additionally required to be bounded.
\end{theorem}

\begin{proof}
Please refer to \citet{ng1999policy} for detailed proof.
\end{proof}

From the above theorem, we can see our intermediate rewards in \eqref{eq:info_gain} is a potential based shaping function and the potential function is $\Phi(s) = - H(y \mid s)$.
For classification task where $y \in \mathcal{Y} = \{1, 2,\dots,K\}$ is a discrete variable, the entropy is naturally bounded, i.e., $0 \leq H(y \mid s) \leq log \vert \mathcal{Y} \vert$, where $\vert \mathcal{Y} \vert$ is the cardinality of the label space.
For regression task where $y \in \mathbb{R}$, the entropy is bounded by $0 \leq H(y \mid s) \leq H(y \mid \emptyset)$. The upper bound $H(y \mid \emptyset)$ is determined by the given surrogate model. Similarly, for the intermediate rewards in \eqref{eq:bpd_gain}, the potential function $\Phi(s) = \frac{\log p(x_u \mid s)}{|u|}$ is also bounded for a given surrogate model.

\begin{algorithm*}[t]
\SetAlgoLined
 1. load pretrained surrogate model $M$, agent $\textit{agent}$ and prediction model $f_\theta(\cdot)$\;
 
 2. instantiate an environment with data $D$: $\textit{env}=\text{Environment}(D)$\;
 
 3. $x_o$, o = \textit{env}.reset(); \tcp{$o=\emptyset$}
 
 4. done = False; reward = 0
 
 \While{not done}{
  aux = $M$.query($x_o$, o); \tcp{query M for auxiliary information}
  \tcp{aux contains the prediction $\hat{y} \sim p(y \mid x_o)$ and output likelihoods,}
  \tcp{the imputed values $\hat{x}_u \sim p(x_u \mid x_o)$ and their uncertainties,}
  \tcp{and estimated utilities $\mathcal{U}_i$ for each $i \in u$ (\eqref{eq:utility}).}
  action = \textit{agent}.act($x_o$, o, aux); \tcp{act based on the state and auxiliary info}
  $r_m$ = $M$.reward($x_o$, $o$, action); \tcp{calculate intermediate rewards with the surrogate model}
  $x_o$, o, done, r = \textit{env}.step(action); \tcp{take a step based on the action}
  \tcp{if action indicates termination: done=True, r=-$\mathcal{L}(\hat{y}(x_o), y)$}
  \tcp{else: done=False, r=$-\alpha \mathcal{C}(action)$, $o=o\cup action$}
  reward = reward + r + $r_m$; \tcp{accumulate rewards}
 }
 prediction = $\textit{agent}$.predict($x_o$, $o$, aux); \tcp{make a final prediction}
 \tcp{using either $M$.predict($x_o$, o, aux) or $f_\theta$($x_o$, o, aux) based on validation}
 \caption{Active Feature Acquisition with \method}
 \label{alg:RL_supp}
\end{algorithm*}

\section{Experiments}\label{sec:experiment}
\subsection{Classification}
For classification tasks, we conduct experiments on MNIST and two UCI datasets. We downsample the MNIST images to $16\times16$ to reduce the total number of features in order to accommodate baselines such as \cite{ma2018eddi}, which had trouble scaling. Features are normalized into the range $[0, 1]$.

The surrogate model for classification task estimate arbitrary conditional distributions that are conditioned on the target variable $y$. For MNIST, we stack conditional coupling transformations and a conditional Gaussian likelihood module. 
For UCI dataset, we use an autoregressive likelihood module. To train the surrogate model, we randomly select two non-overlapping subsets $u$ and $o$ and optimize the arbitrary conditional log likelihood
\begin{equation}
\begin{split}
    &\log p(y, x_u \mid x_o) = \log p(x_u \mid x_o) + \log P(y \mid x_u, x_o)\\
    &= \log p(x_u \mid x_o) + \log \frac{p(x_u,x_o \mid y)P(y)}{\sum_{y'}p(x_u, x_o \mid y')P(y')}.
\end{split}
\end{equation}

The agent is implemented as a PPO policy. Given the current state $x_o$ and the auxiliary information from the surrogate model, we extract a set embedding using set transformer \citep{lee2019set}. The inputs are first transformed to sets by concatenating with the one-hot encoding of their indexes. The set embedding is beneficial to deal with arbitrary dimensionality of the inputs. The policy network then takes the set embedding as inputs and outputs the next action. The critic network takes the same set embedding as inputs and output an estimate of the state values. To help the agent extract meaningful representations from its inputs, we let the prediction model $f_\theta$ take the same set embedding as input. The policy network, the critic network and the prediction function are all implemented as fully connected layers.

We run the baseline model JAFA \citep{shim2018joint} using their public code. We cross-validate the optimal architecture by modifying the number of layers and the size of each layer for both the agent and the classifier.

We adapt EDDI \citep{ma2018eddi} to perform classification task by modifying the decoder to output Categorical distribution for $y$ and Gaussian distribution for $x$. EDDI learns the distribution $p(y, x_o)$ by utilizing a VAE based model. The acquisition metric for EDDI is
\begin{equation}\label{eq:eddi}
\begin{aligned}
    \mathcal{U}_i = &\mathbb{E}_{x_i \sim p(x_i \mid x_o)} \KL[p(z \mid x_i,x_o) \| p(z|x_o)] \\ &- \mathbb{E}_{y,x_i \sim p(y,x_i \mid x_o)} \KL[p(z \mid y,x_i,x_o) \| p(z \mid y,x_o)],
\end{aligned}
\end{equation}
which is estimated using the proposal distribution. Then, a greedy policy that acquires the feature with maximum utility is employed. We similarly cross-validate the architecture for each dataset.

We also compare to a greedy policy using the surrogate model where the utility is calculated by \eqref{eq:cmi_discrete}. At each acquisition step, the one with maximum utility is selected.

\subsection{Regression}
For regression task, the target variable $y$ is concatenated into the features $x$ and the surrogate model learns the distribution $p(y, x_u \mid x_o)$. The agent is similarly implemented as the PPO policy with a set transformer based feature extractor. Baseline models include JAFA and EDDI, where the architecture is selected by cross validation. We also build a greedy policy using our surrogate model by estimating the utility following \eqref{eq:cmi_continuous}.
For \method~and JAFA, the reward for a prediction $\hat{y}$ is calculated as the negative MSE $-\Vert \hat{y}-y\Vert_2^2$.

\subsection{Medical Diagnosis}
We evaluate our model on Physionet challenge 2012 dataset \cite{goldberger2000physiobank}. We first preprocess the dataset by removing some non-relevant features (such as patient ID) and eliminating the instances with very high missing rate (larger than 80\%). The features are then normalized to the range of [0,1]. The model and baselines are mostly the same as the classification experiments for UCI dataset. Since the classses are heavily imbalanced, we use weighted cross entropy as loss and reward. To evaluate the performance for data with missing entries, We first impute those missing features with our GSM model. For EDDI, the missing entries are similarly imputed by the VAE model. JAFA does not have a generative component, thus we simply replace the missing features with zeros. JAFA reports the AUC score for this dataset, but AUC is known inappropriate for imbalanced classification \cite{brabec2018bad}. We instead report the F1 scores for this experiment.

\subsection{Time Series}\label{sec:time_series}
Acquiring features for time series data requires the agent to integrate chronological constraints into the action space. For RL based approach, we manually set the probabilities of invalid action to zeros. For greedy approach, inspired by Thompson sampling \citep{thompson1933likelihood,russo2017tutorial}, we employ a prior distribution to encode our chronological constraint. Specifically, we set the prior as a Dirichlet distribution that is biased towards the selection of earlier time steps:
\begin{equation}
\begin{aligned}
    \pi(\rho) = \text{Dir}\left[\right.&\alpha(T- (\max(o)+1)),\\ &\ldots, \alpha(T-(T-1)) \left.\right](\rho),
\end{aligned}
\end{equation}
where $\alpha$ is a hyperparameter, $T$ is the total time steps, $\max(o)$ represents the latest time step already acquired, and $\rho$ is a distribution for acquisition over the remaining future time steps.
However, we still desire that the acquired features are informative for target $y$. Hence, we update the prior to a posterior using time steps $V$ that are drawn according to how informative they are:
\begin{equation}
\begin{split}
    &p(V_n=t) \propto \exp(I(x_t ; y \mid x_o)),\\
    &t \in \{\max(o)+1,\ldots, T-1 \},\ n \in \{1,\ldots,N\},
\end{split}
\end{equation}
where $N$ is the number of samples.
Due to conjugacy, the posterior is also a Dirichlet distribution
\begin{equation}
\begin{aligned}
    p(\rho \mid V) = \text{Dir}\Bigg[\Bigg.\alpha(T- (\max(o)+1)) \\
    + \sum_{n=1}^{N} \mathbb{I}\{V_n=\max(o)+1\}, \ldots  \Bigg.\Bigg](\rho).
\end{aligned}
\end{equation}
Samples from posterior represent the probabilities of choosing each candidate, which now prefer both earlier time steps and informative features. We draw a sample from posterior and select the most likely time step at each acquisition step.

\subsection{Unsupervised}
To perform active feature acquisition on unsupervised tasks, a.k.a, active instance recognition, we modify the reward for prediction as the negative MSE of the unobserved features, i.e., $-\Vert \hat{x}_u - x_u \Vert_2^2$, where $\hat{x}_u$ is the imputed values of the unobserved features. 

The JAFA is adapted to this task by changing the classifier to an auto-encoder like model, where the observed features $x_o$ are encoded to predict the unobserved features $x_u$. 

For EDDI, by plugging $y=x$ into \eqref{eq:eddi}, we have the acquisition metric for this setting as
\begin{equation}
    \mathcal{U}_i = \mathbb{E}_{x_i \sim p(x_i \mid x_o)} \KL[p(z \mid x_i,x_o) \| p(z|x_o)],
\end{equation}
since the second KL term in \eqref{eq:eddi} equals to zero.

To build a greedy policy using our surrogate model, we estimate the utility using \eqref{eq:cmi_air}. Monte Carlo estimation is utilized to estimate the entropy.

\section{Hyperparameters}
We search the hyperparameters for both our \method~and baselines using cross-validation. The range of the hyperparameters is listed in Table~\ref{tab:hyperparameter}.

\input{hyperparameter}

\section{Additional Results}\label{sec:results}
Due to the space limit, we only show one example for the acquisition process in the main text. Figure~\ref{fig:supp_mnist_afa} and \ref{fig:supp_mnist_air} show some additional examples for AFA and AIR tasks respectively.
In Fig.~\ref{fig:supp_mnist_afa_greedy} and \ref{fig:supp_mnist_air_greedy}, we present several examples of the acquisition process from the greedy policy. Note that the predictions for both the greedy and the non-greedy policy are from the same pretrained arbitrary conditioning model, therefore the only difference is the acquired features. Comparing the greedy and the non-greedy policy suggests that the non-greedy policy eliminates the prediction uncertainty much faster than the greedy one.

\begin{figure*}
    \begin{minipage}{0.485\textwidth}
    \centering
    \includegraphics[width=0.98\textwidth]{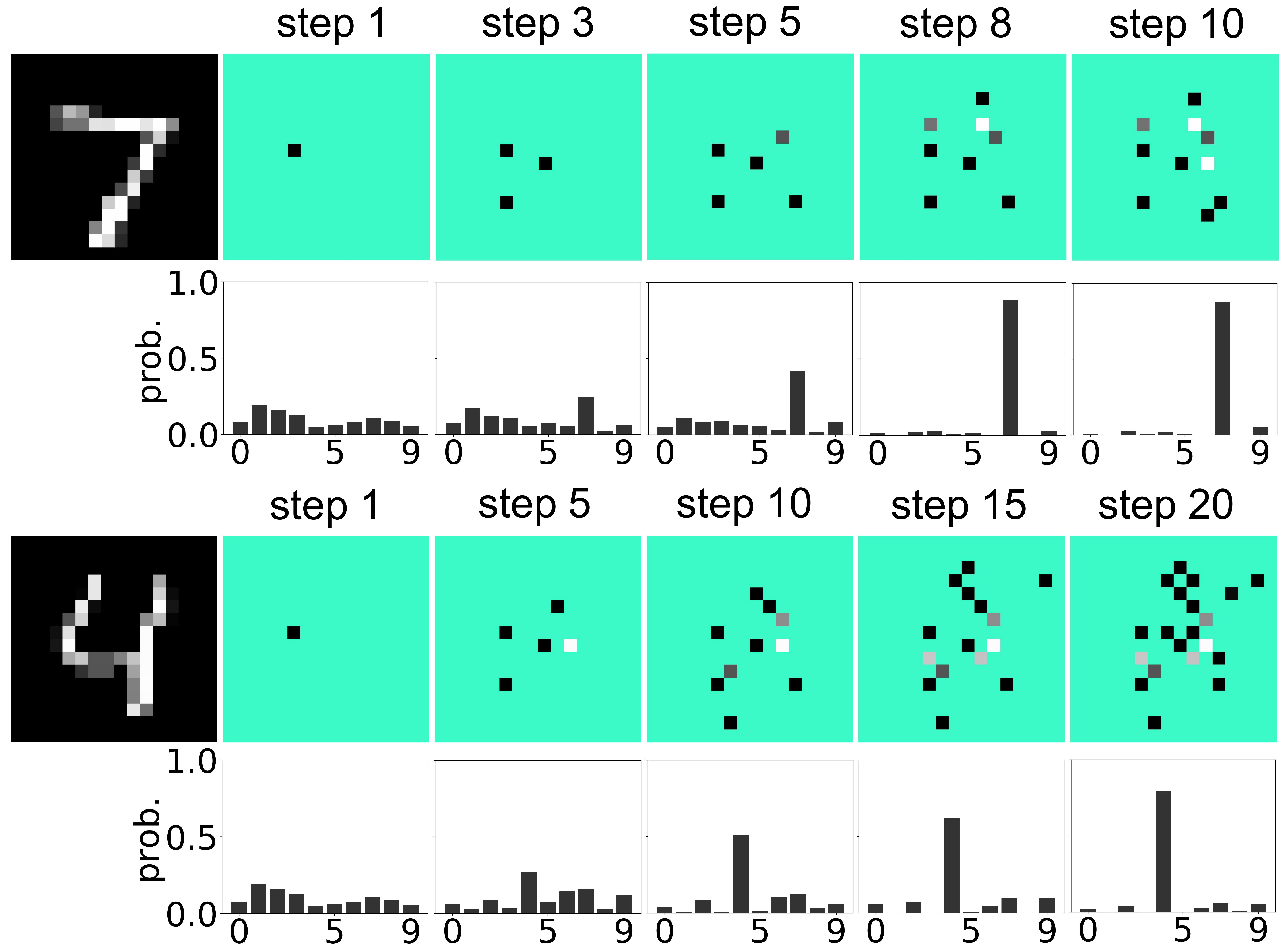}
    \caption{Examples of the acquisition process for AFA task from GSMRL.}
    \label{fig:supp_mnist_afa}
    \end{minipage}
    \quad
    \begin{minipage}{0.485\textwidth}
    \centering
    \includegraphics[width=0.98\textwidth]{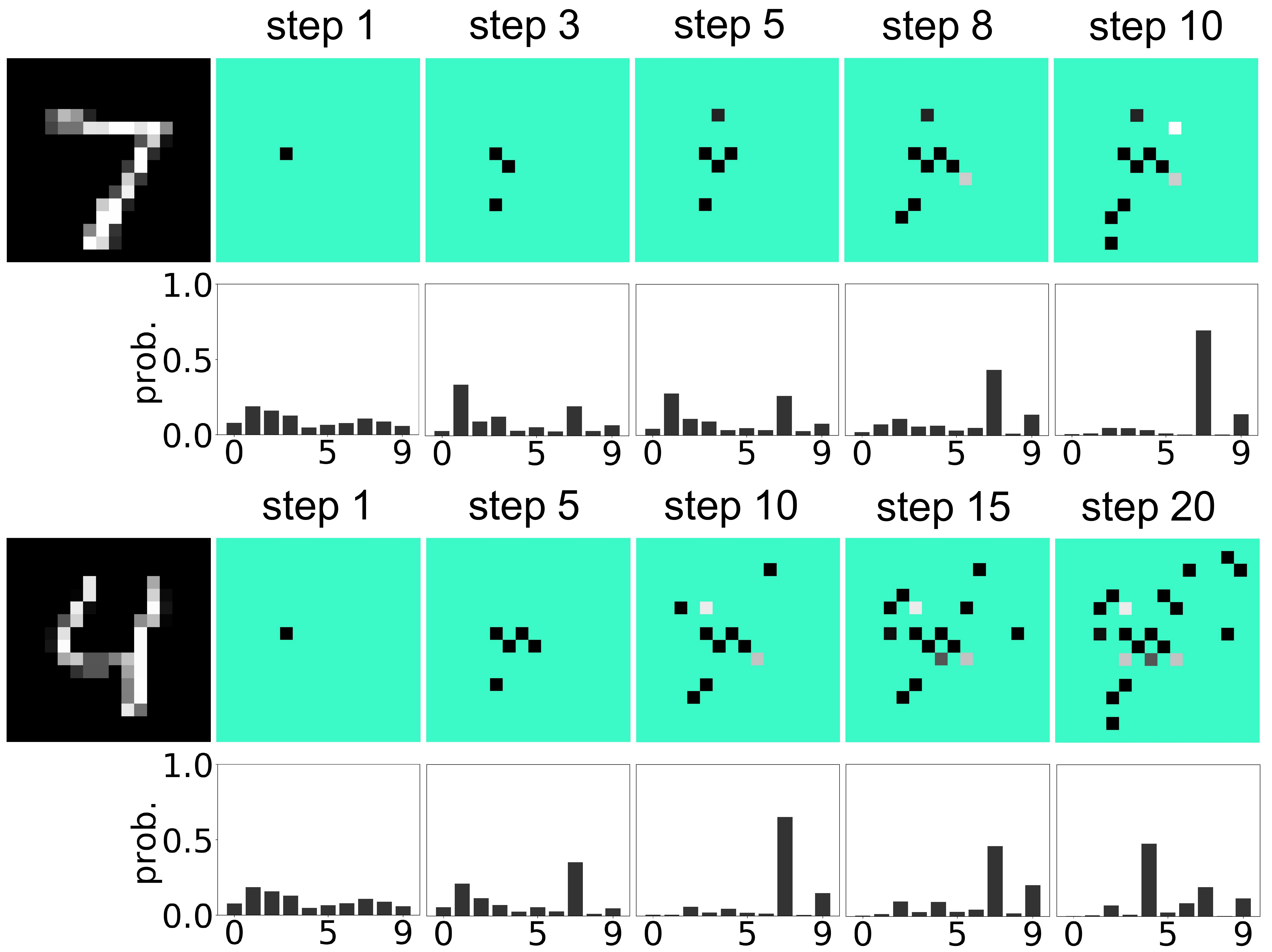}
    \caption{Examples of the acquisition process for AFA task from GSM+Greedy.}
    \label{fig:supp_mnist_afa_greedy}
    \end{minipage}
\end{figure*}

\begin{figure*}
    \begin{minipage}{0.485\textwidth}
    \centering
    \includegraphics[width=0.98\textwidth]{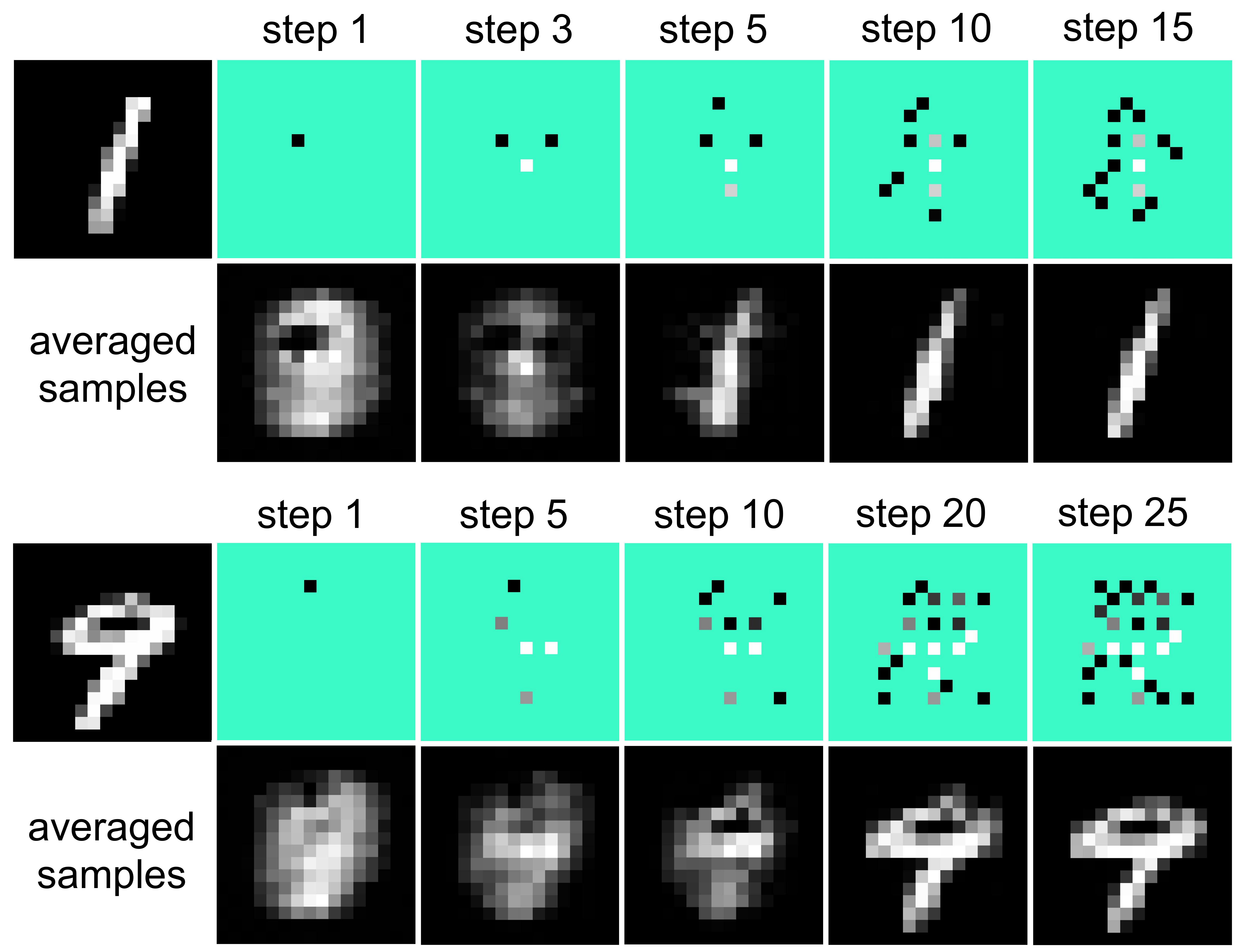}
    \caption{Examples of the acquisition process for AIR task from GSMRL.}
    \label{fig:supp_mnist_air}
    \end{minipage}
    \quad
    \begin{minipage}{0.485\textwidth}
    \centering
    \includegraphics[width=0.98\textwidth]{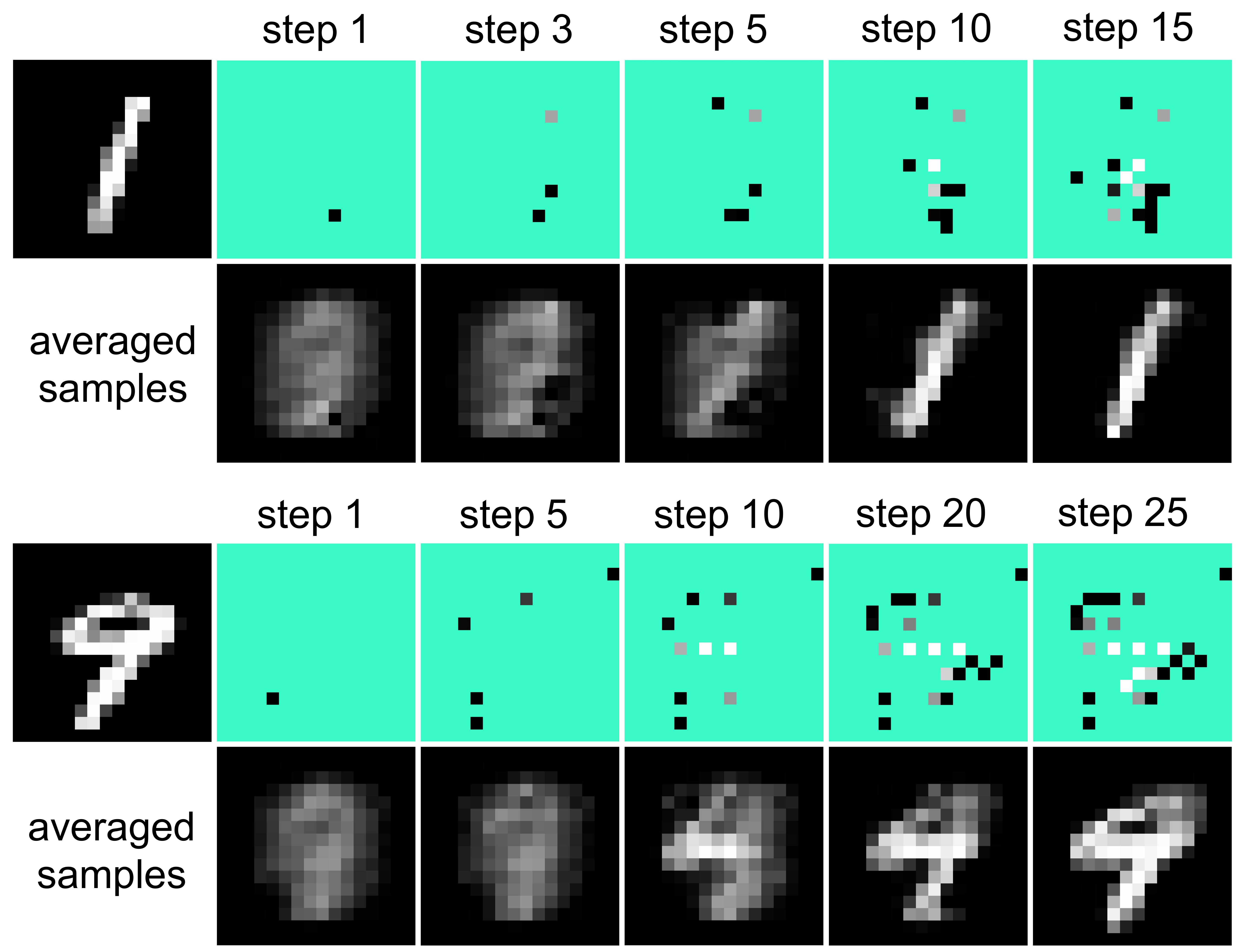}
    \caption{Examples of the acquisition process for AIR task form GSM+Greedy.}
    \label{fig:supp_mnist_air_greedy}
    \end{minipage}
\end{figure*}

In Fig.~\ref{fig:mnist_cls} and \ref{fig:mnist_air}, we present the acquired features from our GSMRL for several testing examples. To better understand the overall distribution of the acquired features across all the testing instances, we plot the frequencies of each feature being acquired in Fig.~\ref{fig:dist_afa} and \ref{fig:dist_air} for both AFA and AIR on MNIST respectively. A higher value of the frequency means the corresponding feature is acquired for more testing instances. Specifically, the frequency for a feature equals to one means the corresponding feature is a common feature acquired for all testing instances. The frequency loosely represents the importance of each feature, which could help with model interpretation and reasoning about decision making. We will explore this direction in future works.

\begin{figure*}
    \begin{minipage}{0.48\textwidth}
    \centering
    \includegraphics[width=0.98\textwidth]{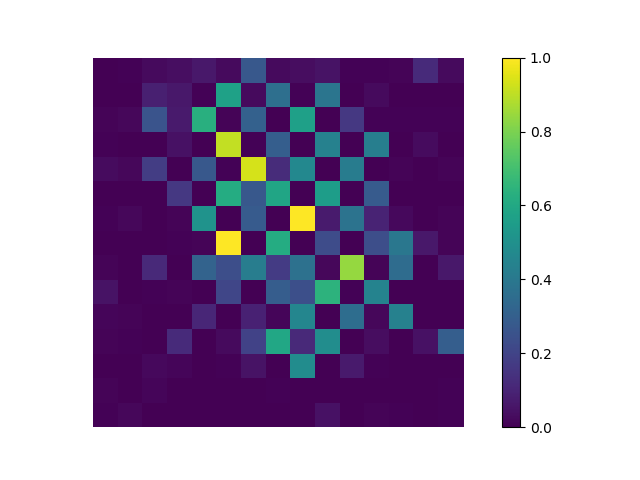}
    \caption{Acquisition frequency for AFA.}
    \label{fig:dist_afa}
    \end{minipage}
    \quad
    \begin{minipage}{0.48\textwidth}
    \centering
    \includegraphics[width=0.98\textwidth]{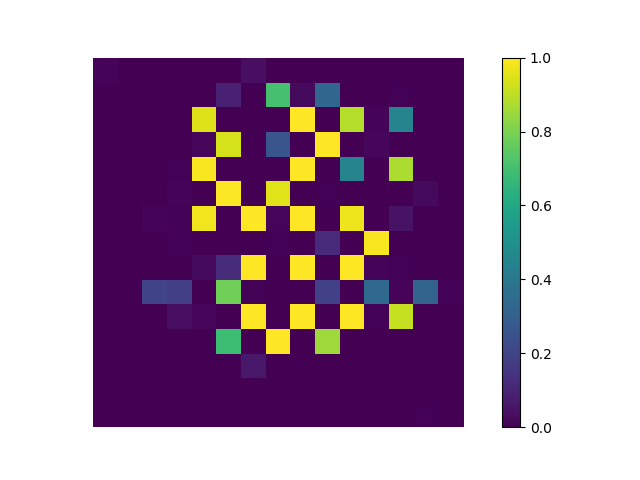}
    \caption{Acquisition frequency for AIR.}
    \label{fig:dist_air}
    \end{minipage}
\end{figure*}

In Fig.~\ref{fig:sensitivity}, we analyse the sensitivity of our model to random initialization by running our model three times independently with different random seeds. We report the mean and standard deviation for both the number of acquisitions and the task performance. Baseline performance are presented for reference. We can see that our model is robust to random initialization and performs consistently better than baselines.

\begin{figure*}
    \centering
    \subfigure[Classification]{
    \includegraphics[width=0.3\textwidth]{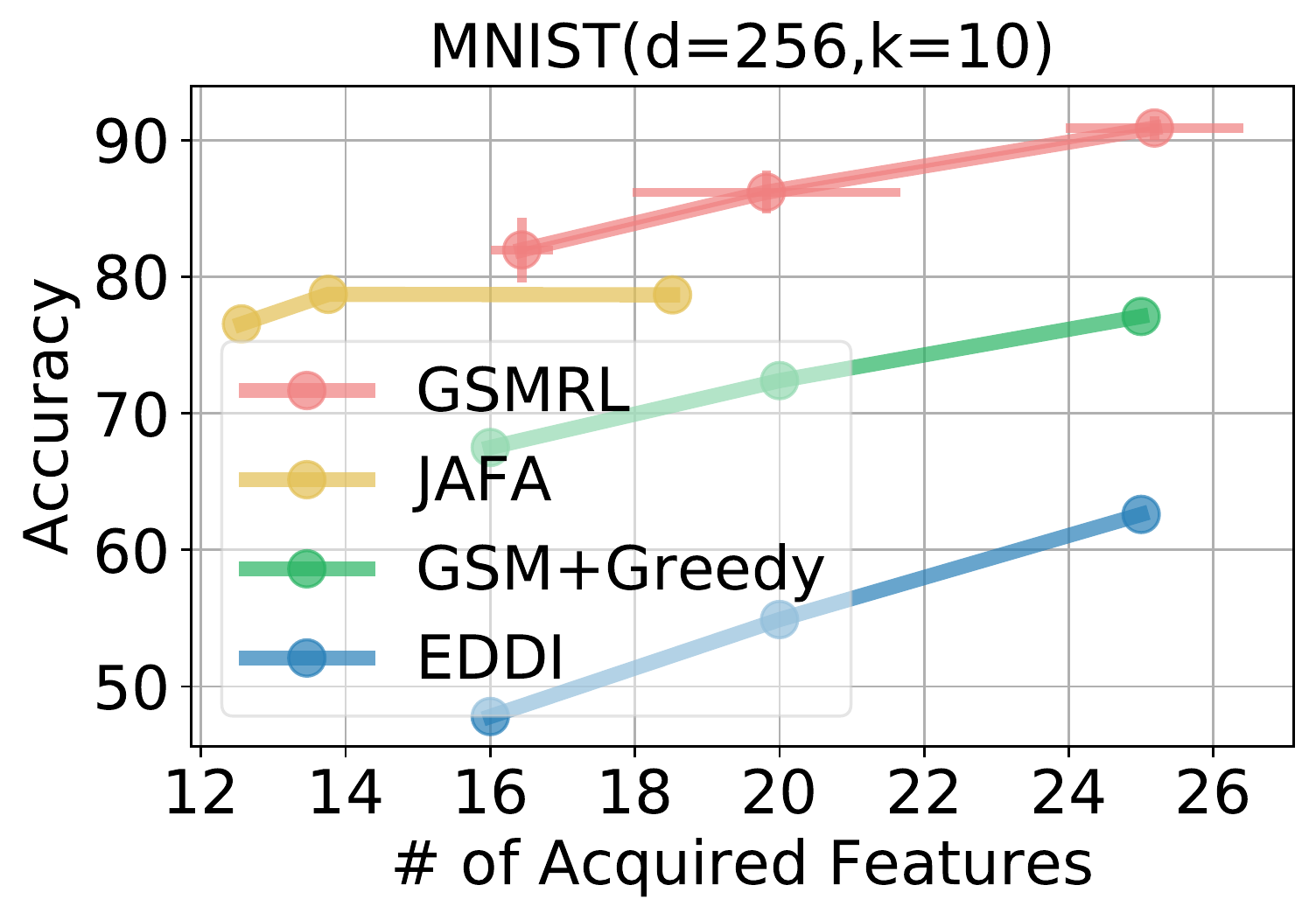}
    \includegraphics[width=0.3\textwidth]{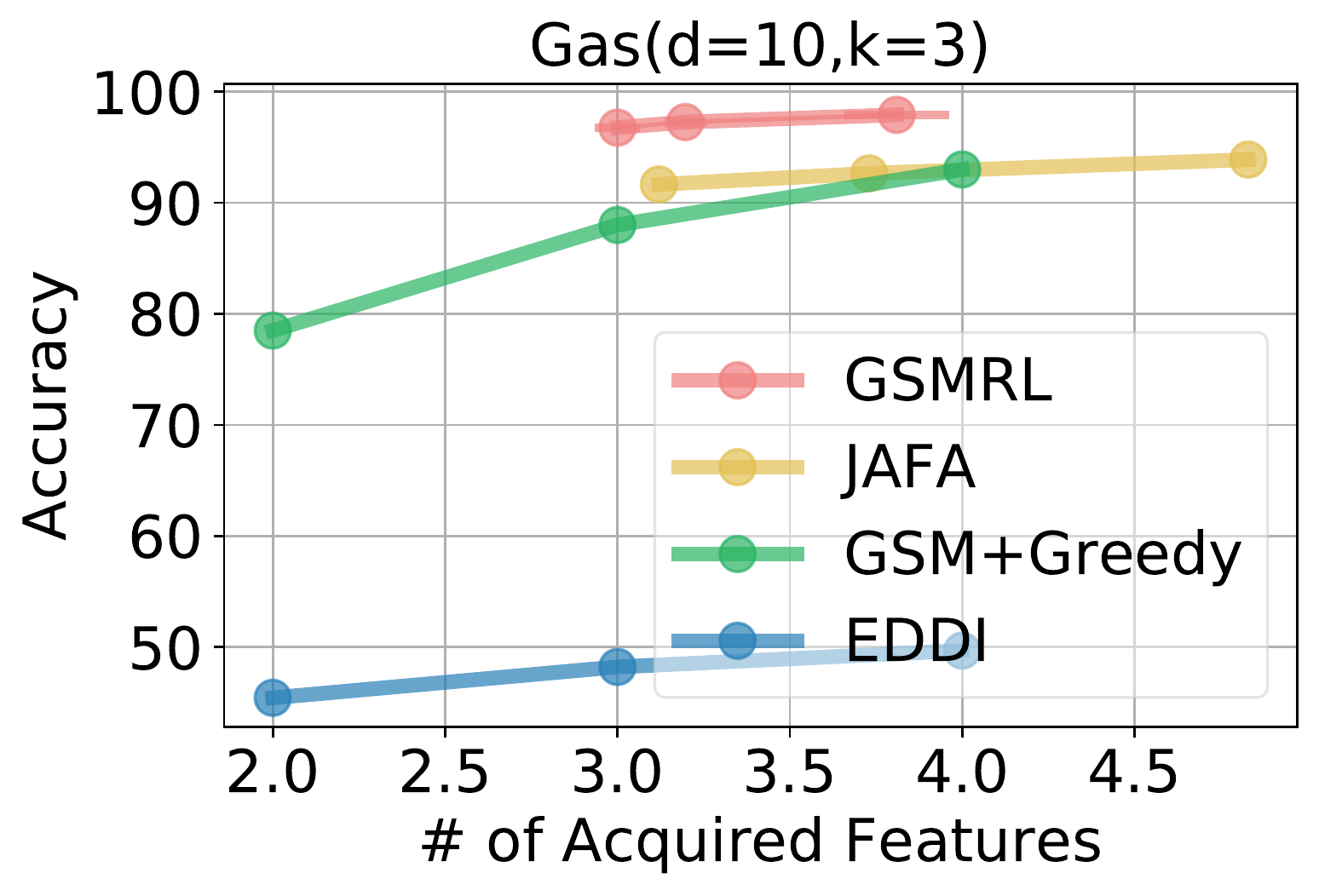}
    \includegraphics[width=0.3\textwidth]{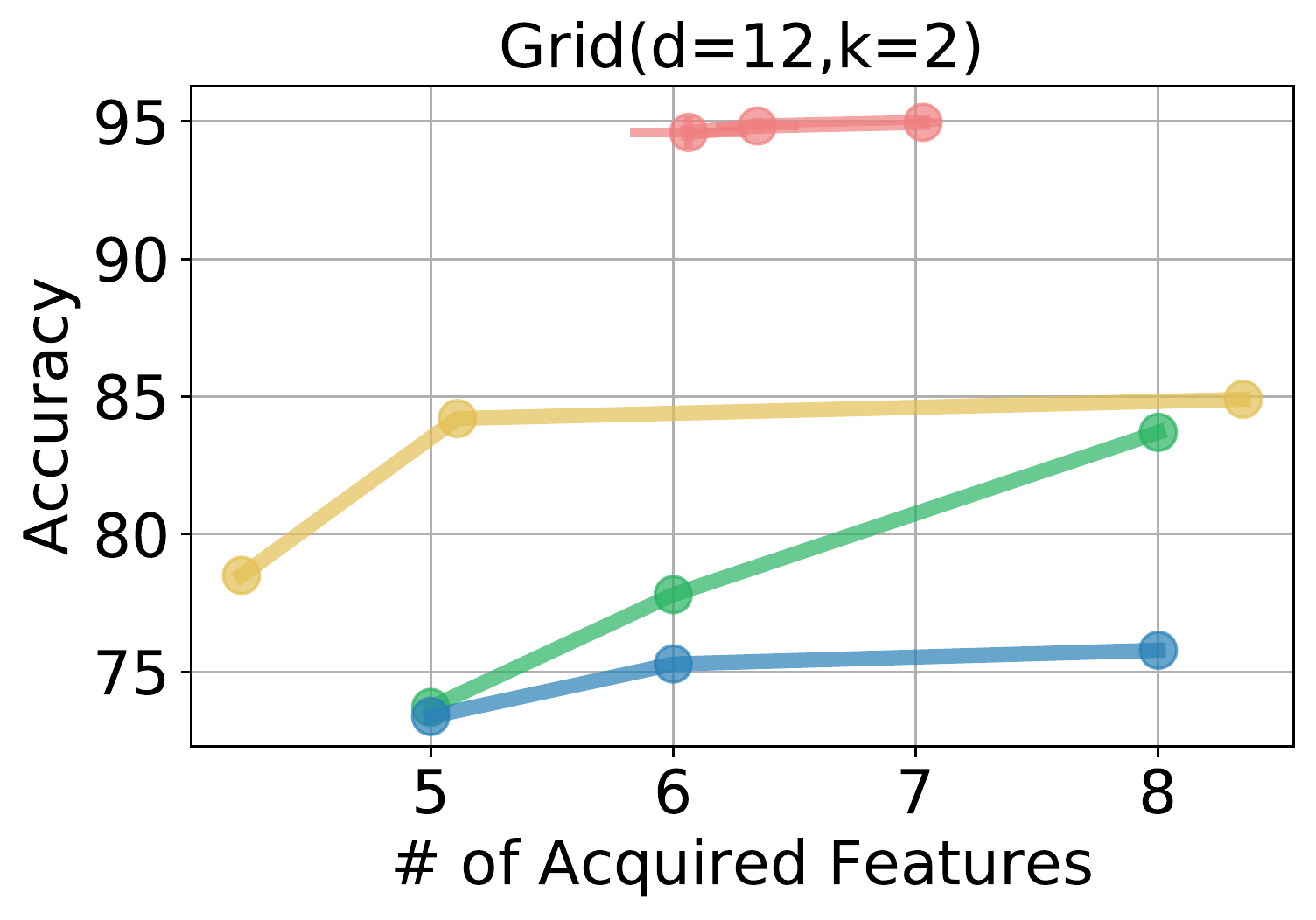}}
    \subfigure[Regression]{
    \includegraphics[width=0.3\textwidth]{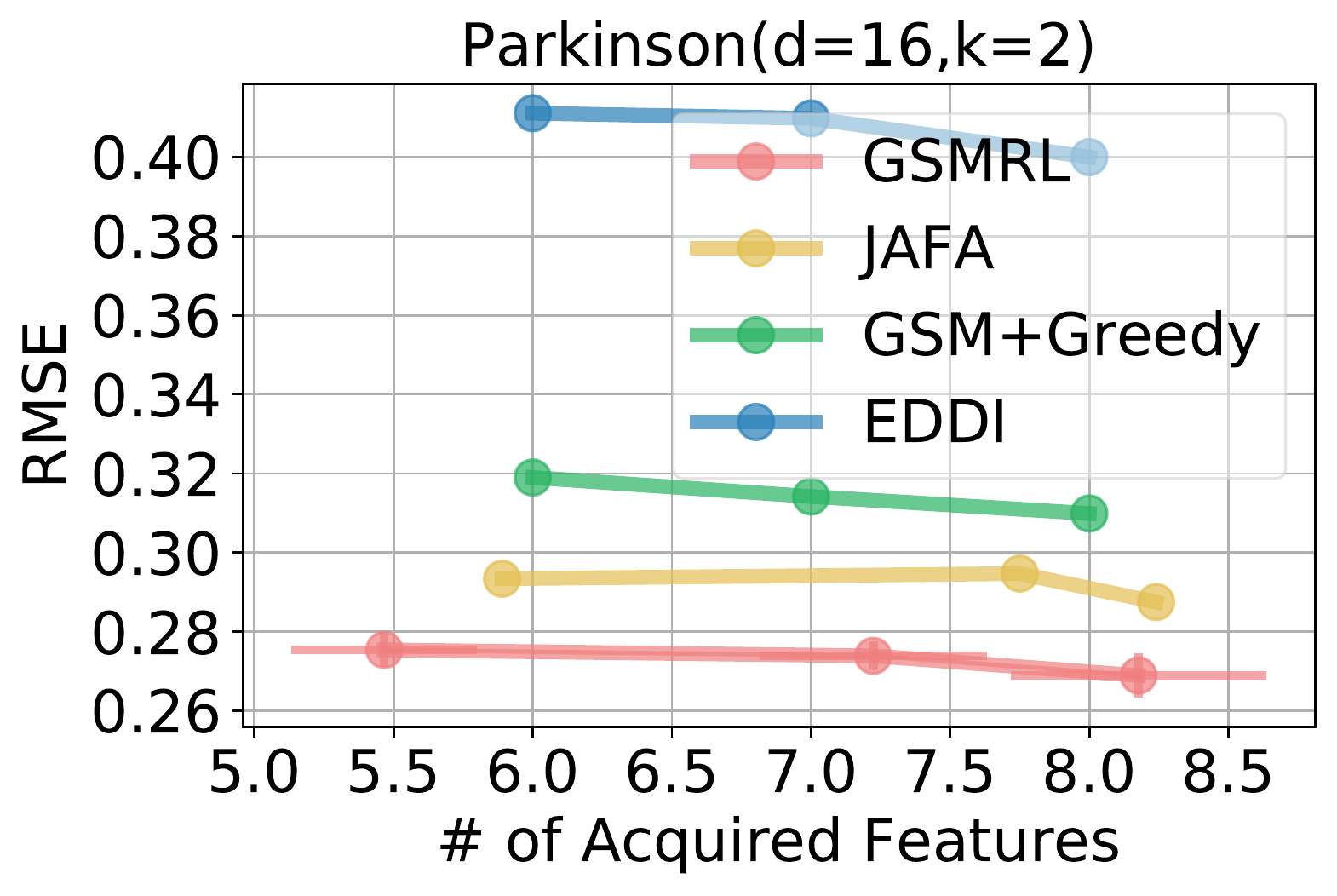}
    \includegraphics[width=0.3\textwidth]{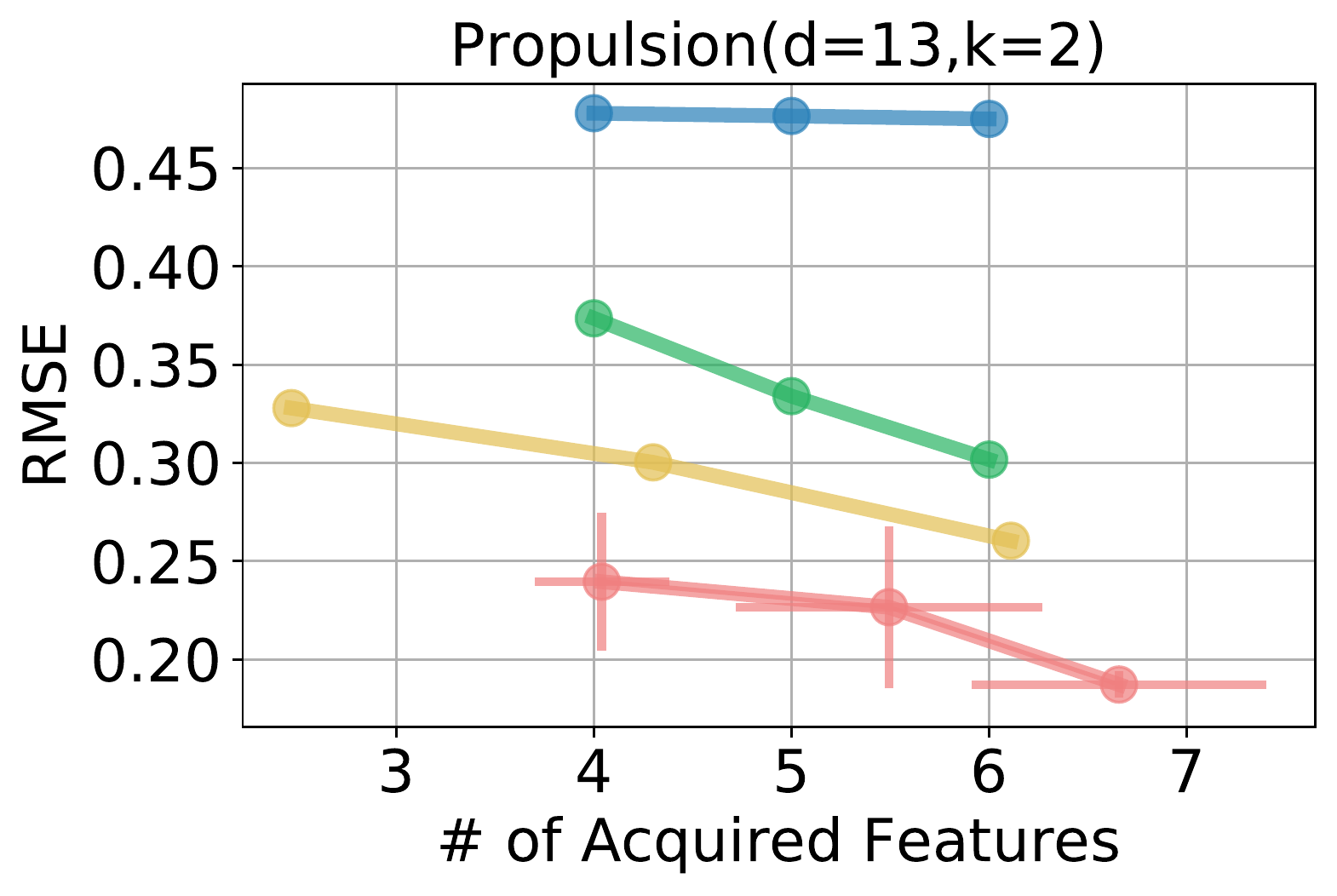}}
    \subfigure[Time Series]{
    \includegraphics[width=0.3\textwidth]{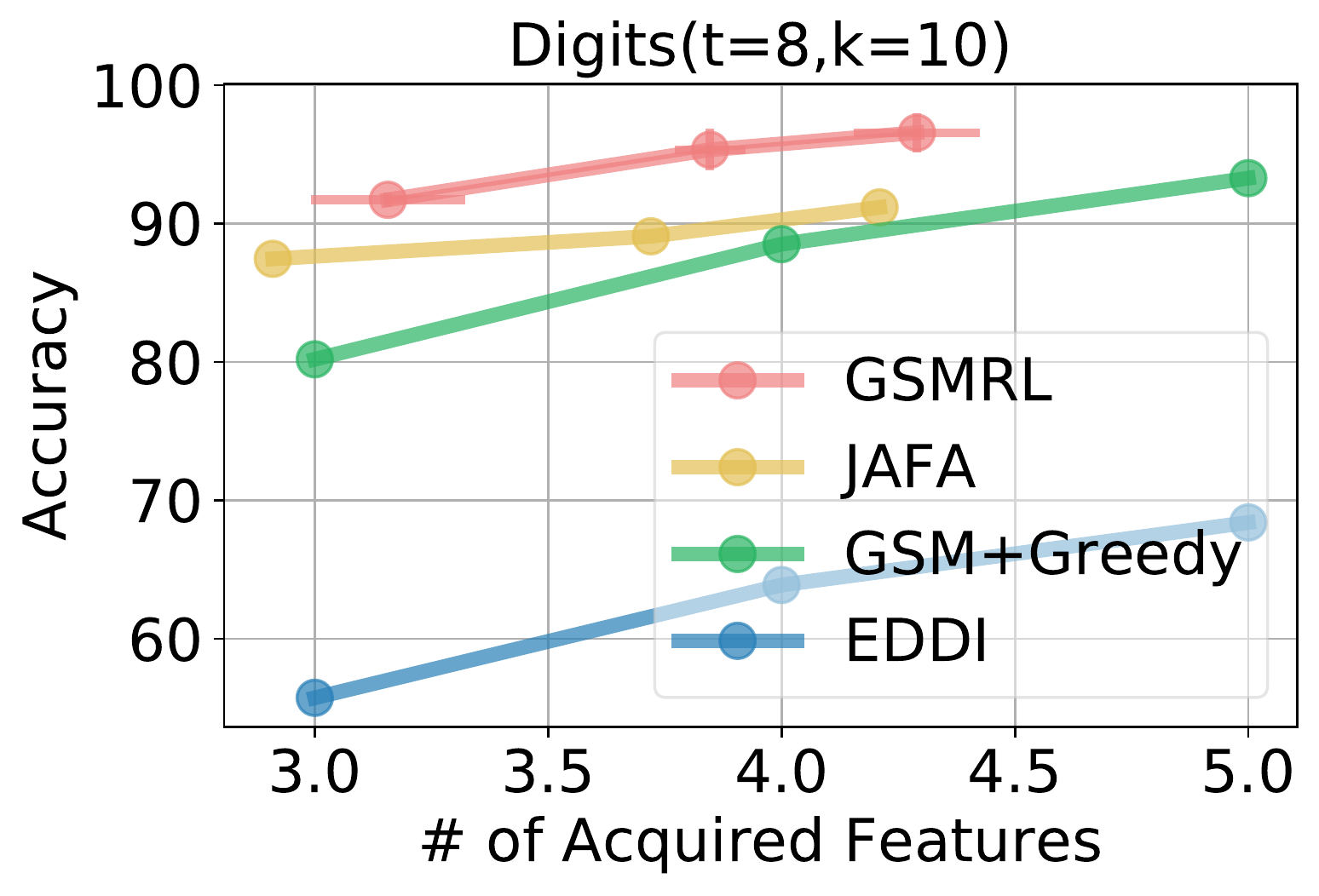}
    \includegraphics[width=0.3\textwidth]{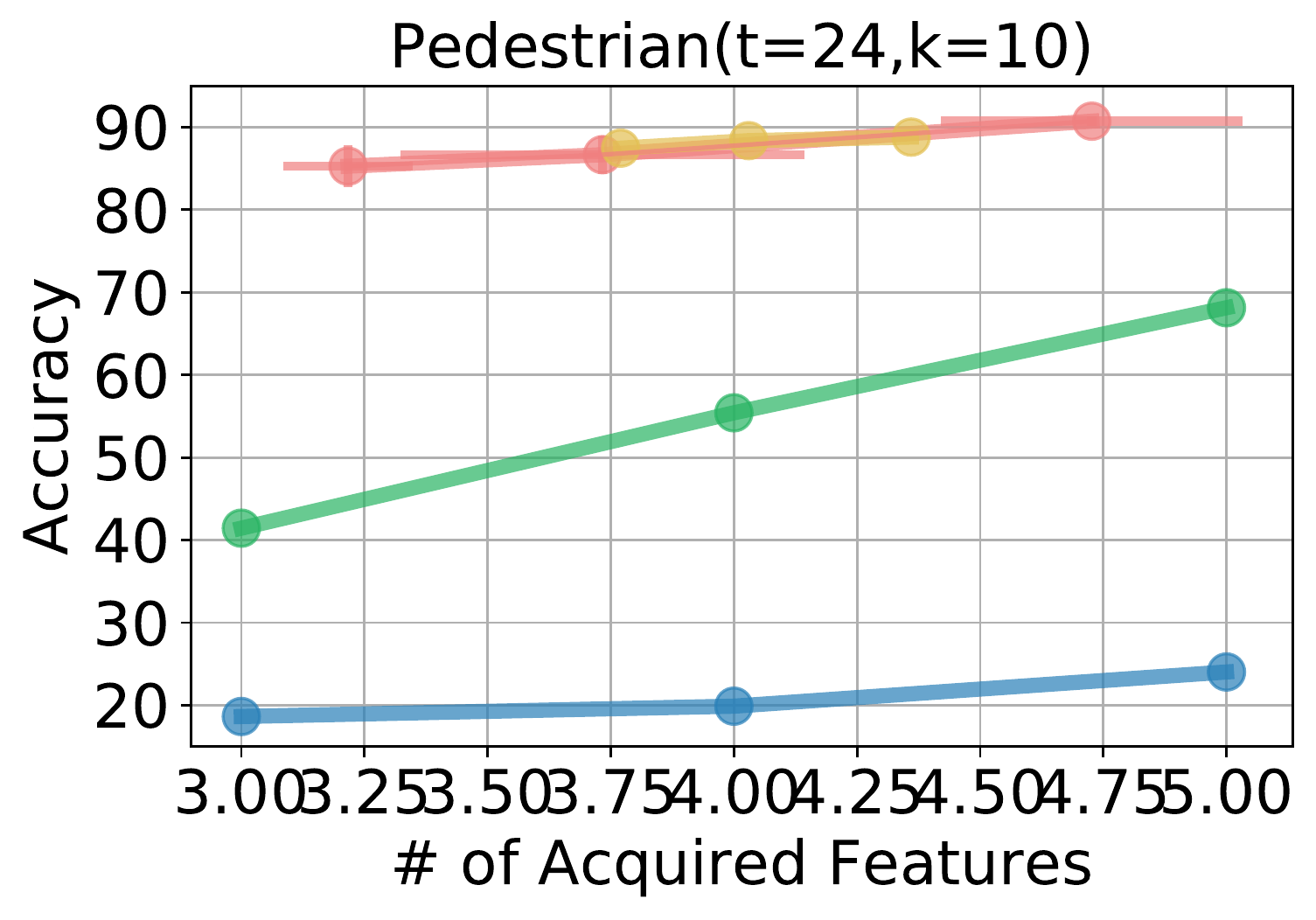}}
    \subfigure[Unsupervised]{
    \includegraphics[width=0.3\textwidth]{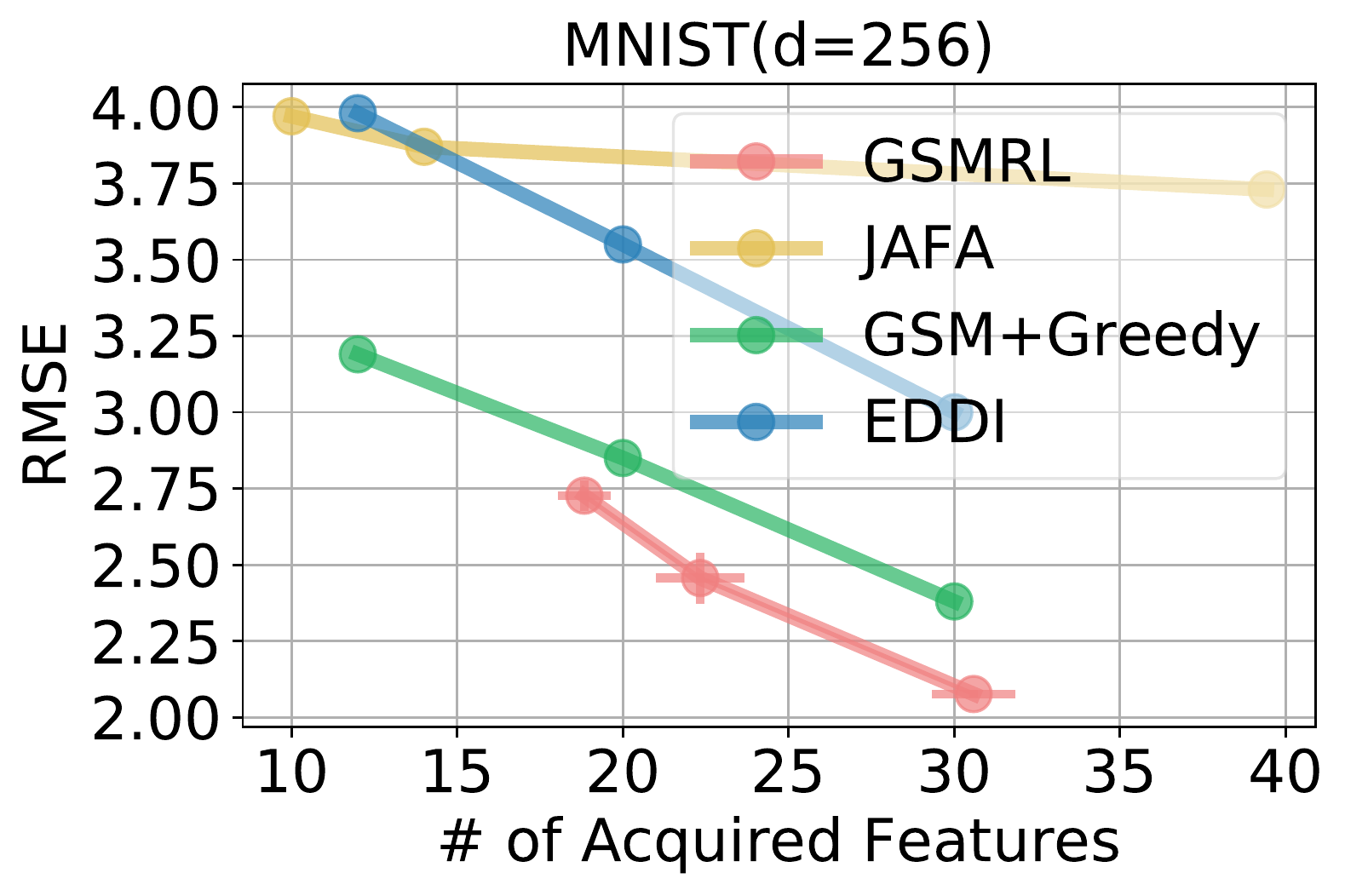}
    \includegraphics[width=0.3\textwidth]{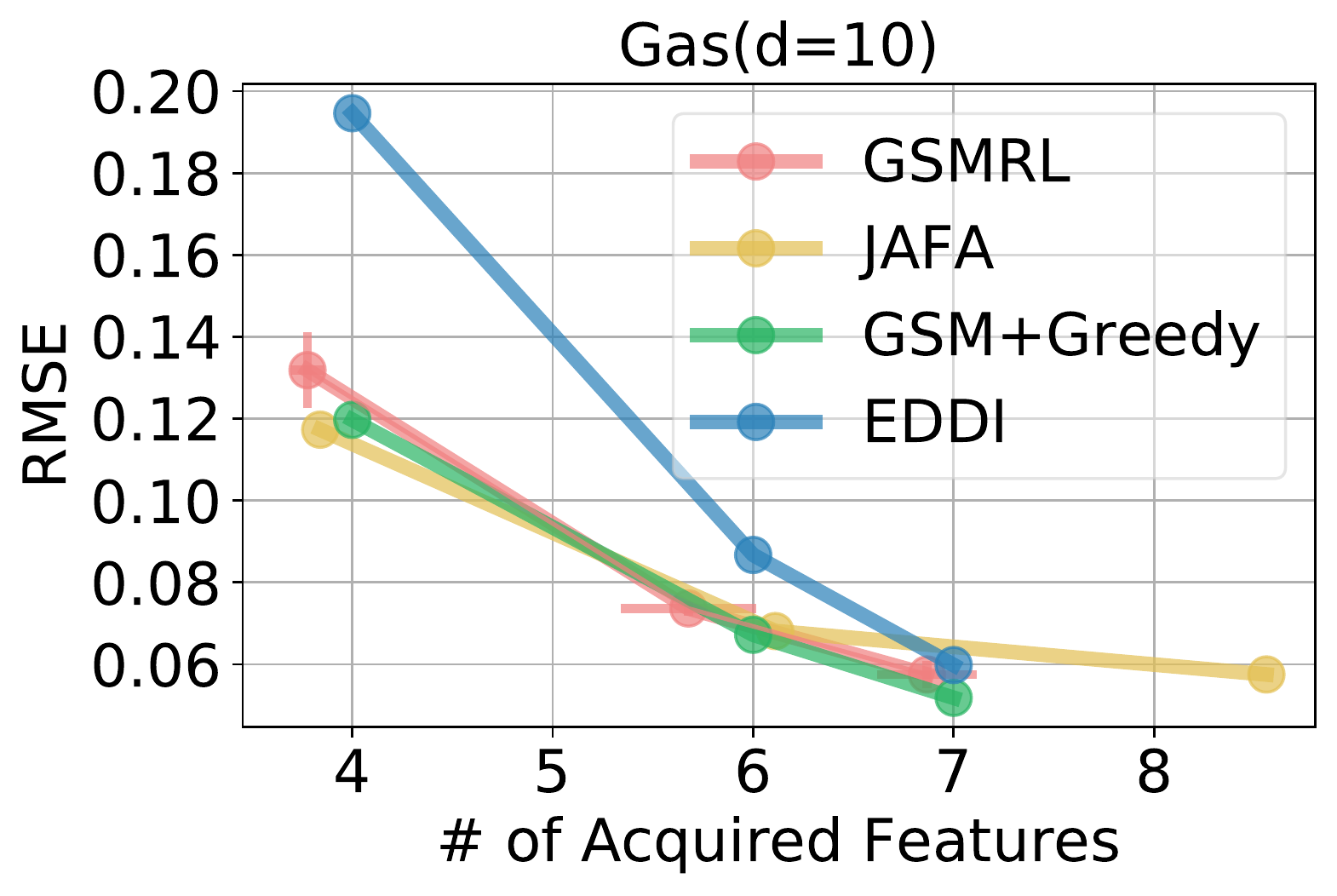}
    \includegraphics[width=0.3\textwidth]{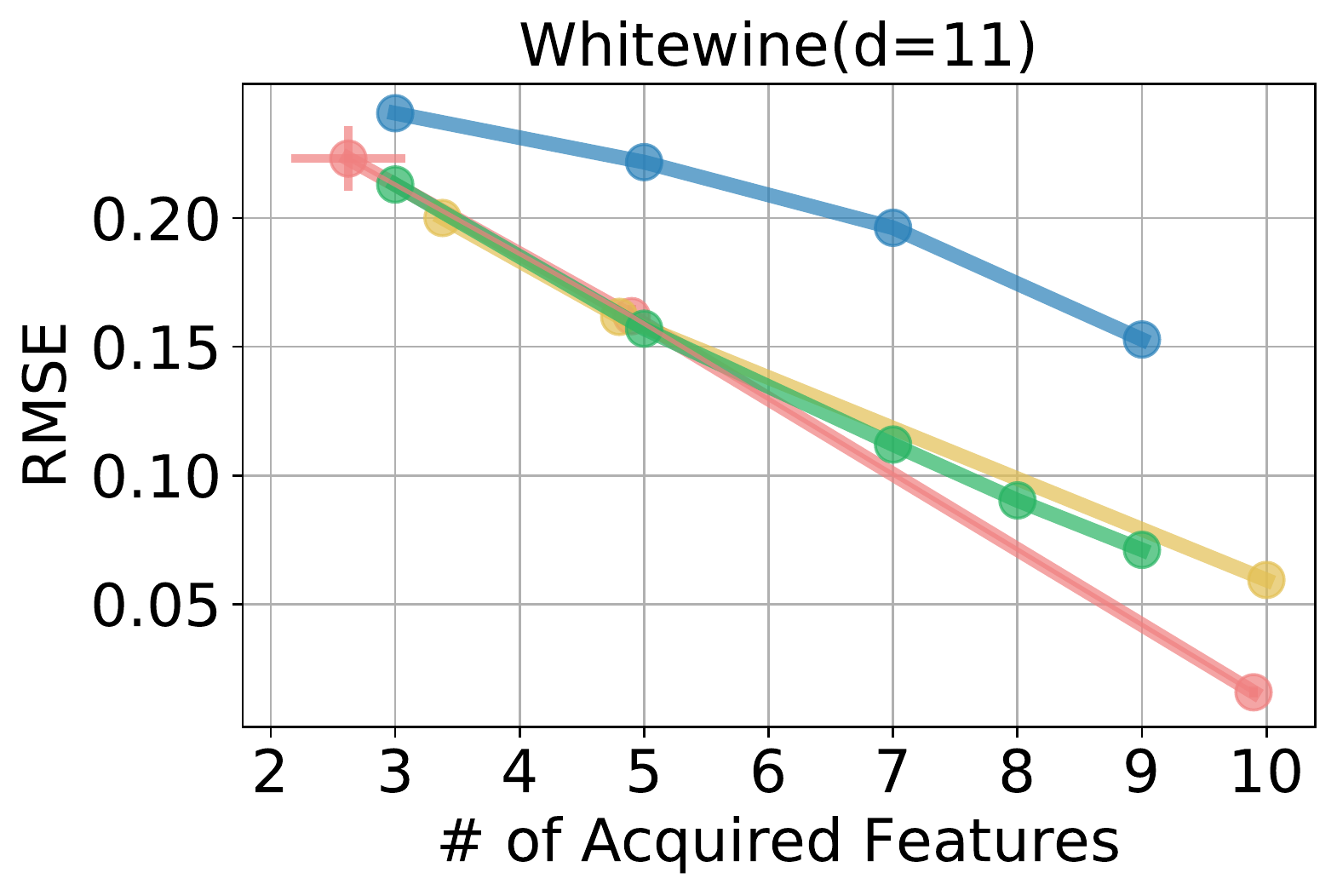}}
    \caption{Sensitivity analysis by running multiple times independently. Mean and standard deviation are reported for both the number of acquisitions and task performance.}
    \label{fig:sensitivity}
\end{figure*}

\textbf{Small vs. Large Action Space}
For the sake of comparison, we employ a downsampled version of MNIST in the experiment section. Here, we show that our GSMRL model can be easily scaled up to a large action space. We conduct experiments using the original MNIST of size $28 \times 28$. We observe that JAFA has difficulty in scaling to this large action space, the agent either acquires no feature or acquires all features. 
\begin{wrapfigure}{r}{0.5\linewidth}
\includegraphics[width=\linewidth]{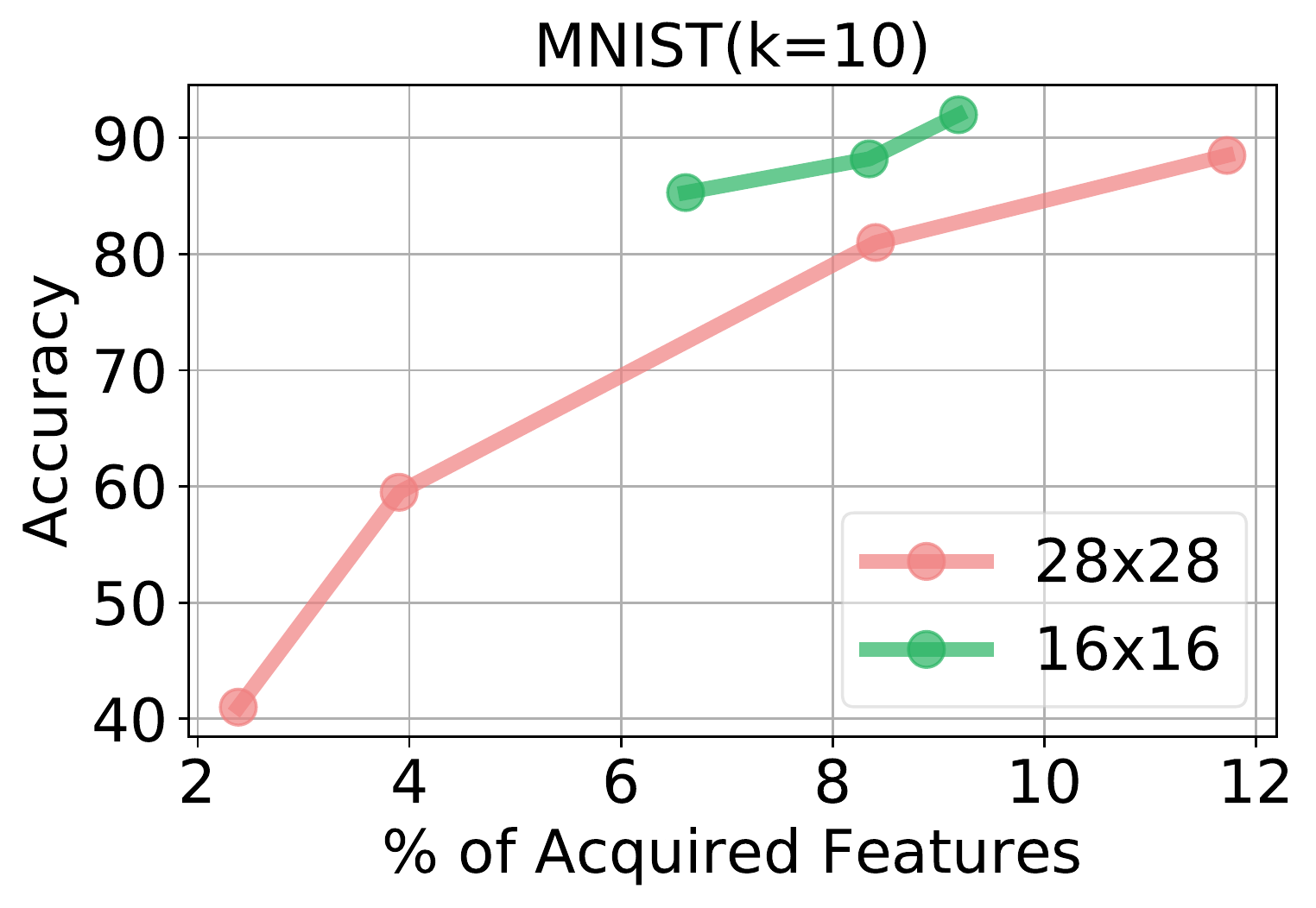}
\vspace{-20pt}
\caption{Acquisition with large action space.}
\label{fig:mnist_cls_large}
\end{wrapfigure}
The greedy approaches are also hard to scale, since at each acquisition step, the greedy policy will need to compute the utilities for every unobserved features, which incurs a total $O(d^2)$ complexity. In contrast, our GSMRL only has $O(d)$ complexity. Furthermore, with the help of the surrogate model, our GSMRL is pretty stable during training and converges to the optimal policy quickly. Fig.~\ref{fig:mnist_cls_large} shows the accuracy with a certain percent of features acquired. The task is definitely harder for large action space as can be seen from the drop in performance when the agent acquires the same percentage of features for both small and large action space, but our GSMRL still achieves high accuracy by only acquiring a small portion of features.

\textbf{Reward Evaluation}
Since we are dealing with a dynamic acquisition scenario, different algorithms or the same algorithm with different hyperparameters, such as $\alpha$, could lead to different acquisitions, which renders the direct comparison difficult. In the experiment section, we compare different algorithms by plotting the performance curve w.r.t. the number of acquisitions. Here, we utilize another evaluation metric that directly compares the returned reward. 
\begin{wraptable}{r}{0.23\textwidth}
\small
\caption{Normalized rewards for MNIST AFA experiments.}
\label{tab:reward}
\begin{tabular}{cc}
\toprule
Algorithm  & Reward  \\
\midrule
GSMRL      & 0.7998  \\
JAFA       & 0.7335  \\
GSM+Greedy & 0.7038  \\
EDDI       & 0.6116  \\
\bottomrule
\end{tabular}
\end{wraptable}
We use a normalized reward for evaluation where for a $d$-dimensional instance, each feature costs $\frac{1}{d}$ and the final classification is rewarded 1 if the prediction is correct, otherwise the reward is zero. The normalized reward is within the rage of $[-1,1]$ where a correct classification with no feature acquired obtains the highest reward 1, a wrong classification with all feature acquired obtains the lowest reward -1, and a correct classification with all features acquired obtains the reward 0. We report the normalized reward for MNIST classification in Table \ref{tab:reward}.

\end{document}

%% file: hyperparameter.tex
\begin{table}[h]
\centering
\small
\caption{Hyperparameters for \method~and baselines.}
\label{tab:hyperparameter}
\resizebox{\linewidth}{!}{
\begin{tabular}{|l|l|l|}
\hline
\multirow{9}{*}{GSMRL} & set transformer          & $\{32,64\}\times \{1,2\}$             \\ \cline{2-3}
                       & set embedding size       & ${\{32,64\}}$                         \\ \cline{2-3}
                       & policy network           & $\{32,64\} \times \{2,3\}$            \\ \cline{2-3} 
                       & critic network           & $\{32,64\} \times \{2,3\}$            \\ \cline{2-3} 
                       & prediction network       & $\{64,128\} \times \{2,3\}$           \\ \cline{2-3} 
                       & advantage $\lambda$      & 0.95                              \\ \cline{2-3} 
                       & discount factor $\gamma$ & 0.99                              \\ \cline{2-3} 
                       & PPO clip range           & $[0.8,1.2]$                       \\ \cline{2-3} 
                       & entropy coefficient      & 0.0                               \\ \hline
\multirow{3}{*}{JAFA}  & set embedding size       & $\{16,32,64,128\}$                  \\ \cline{2-3} 
                       & Q network                & $\{16,32,64,128\} \times \{2,3,4,5\}$ \\ \cline{2-3} 
                       & prediction network       & $\{16,32,64,128\}\times \{2,3,4,5\}$  \\ \hline
\multirow{4}{*}{EDDI}  & set embedding size       & $\{10,20,50,100\}$                  \\ \cline{2-3} 
                       & encoder                  & $\{32,64,128,256\}\times\{3,4,5,6\}$  \\ \cline{2-3} 
                       & latent code              & $\{10,20,50,100\}$                  \\ \cline{2-3} 
                       & decoder                  & $\{32,64,128,256\}\times\{3,4,5,6\}$  \\ \hline
\end{tabular}}
\end{table}